%% file: main.tex
\documentclass{article}

\usepackage{microtype}
\usepackage{graphicx}
\usepackage{subfigure}
\usepackage{booktabs} 

\usepackage{hyperref}

\usepackage[accepted]{icml2023}

\usepackage{amsmath}
\usepackage{amssymb}
\usepackage{mathtools}
\usepackage{amsthm}

\usepackage[capitalize,noabbrev]{cleveref}

\theoremstyle{plain}
\newtheorem{theorem}{Theorem}[section]

\theoremstyle{definition}

\theoremstyle{remark}

\usepackage[textsize=tiny]{todonotes}

\usepackage{hyperref}
\usepackage{url}
\usepackage{multirow}
\usepackage{multicol}
\usepackage{graphics}
\usepackage{lipsum}
\usepackage{adjustbox}
\usepackage{wrapfig}
\usepackage{enumitem}
\usepackage{graphicx}
\usepackage{booktabs}
\usepackage{caption}
\setlist[itemize]{leftmargin=10mm}
\usepackage{bm}

\usepackage{latexsym}
\usepackage{amsmath}
\usepackage{amssymb}
\usepackage{amsthm}
\usepackage{epsfig}
\usepackage{hyperref}
\usepackage{tikz}
\usepackage{tikz-3dplot}
\usetikzlibrary{perspective}

\DeclareMathOperator*{\Expt}{\mathbb{E}}

\newcommand{\hypothesisclass}{\mathcal{F}}

\newcommand{\setfeaturegen}{\mathcal{H}}
\newcommand{\featuregen}{H}

\newcommand{\numfeature}{\mathcal{X}^{\text{num}}}
\newcommand{\catfeature}{\mathcal{X}^{\text{cat}}}
\newcommand{\popunumfeature}{\mathcal{Q}}
\newcommand{\Dtrain}{\mathcal{D}_{\text{train}}}
\newcommand{\Dtest}{\mathcal{D}_{\text{test}}}
\newcommand{\ktrain}{k_1}
\newcommand{\ktest}{k_2}
\newcommand{\nmin}{h}
\newcommand{\group}{\mathsf{g}}

\newcommand{\Distr}{\mathcal{P}}

\usepackage{amssymb}
\usepackage{pifont}
\newcommand{\xmark}{{\color{red} \ding{55}}}
\usepackage{pgfplots}
\usepackage{tikz-3dplot}

\icmltitlerunning{OpenFE: Automated Feature Generation with Expert-level Performance}

\begin{document}

\twocolumn[
\icmltitle{OpenFE: Automated Feature Generation with Expert-level Performance}



\setlength{\itemsep}{-1em}

\begin{icmlauthorlist}
\icmlauthor{Tianping Zhang}{thu}
\icmlauthor{Zheyu Zhang}{thu}
\icmlauthor{Zhiyuan Fan}{thu}
\icmlauthor{Haoyan Luo}{cuhk}
\icmlauthor{Fengyuan Liu}{uw}
\icmlauthor{Qian Liu}{sea}
\icmlauthor{Wei Cao}{msra}
\icmlauthor{Jian Li}{thu}
\end{icmlauthorlist}

\icmlaffiliation{thu}{Institute for Interdisciplinary Information Sciences (IIIS), Tsinghua University, Beijing, China}
\icmlaffiliation{cuhk}{School of Data Science, The Chinese University of Hong Kong (Shenzhen), Shenzhen, China}
\icmlaffiliation{msra}{Microsoft Research Asia, Beijing, China}
\icmlaffiliation{sea}{Sea AI Lab, Singapore}
\icmlaffiliation{uw}{Paul G. Allen School of Computer Science \& Engineering, University of Washington, Seattle, U.S.}

\icmlcorrespondingauthor{Jian Li}{lijian83@mail.tsinghua.edu.cn}

\icmlkeywords{Tabular Data, Automated Feature Generation}

\vskip 0.3in
]



\printAffiliationsAndNotice{}  

\begin{abstract}
The goal of automated feature generation is to liberate machine learning experts from the laborious task of manual feature generation, which is crucial for improving the learning performance of tabular data. The major challenge in automated feature generation is to efficiently and accurately identify effective features from a vast pool of candidate features. In this paper, we present OpenFE, an automated feature generation tool that provides competitive results against machine learning experts. OpenFE achieves high efficiency and accuracy with two components: 1) a novel feature boosting method for accurately evaluating the incremental performance of candidate features and 2) a two-stage pruning algorithm that performs feature pruning in a coarse-to-fine manner. Extensive experiments on ten benchmark datasets show that OpenFE outperforms existing baseline methods by a large margin. We further evaluate OpenFE in two Kaggle competitions with thousands of data science teams participating. In the two competitions, features generated by OpenFE with a simple baseline model can beat 99.3\% and 99.6\% data science teams respectively. In addition to the empirical results, we provide a theoretical perspective to show that feature generation can be beneficial in a simple yet representative setting.
\end{abstract}

\section{Introduction}
Feature generation is an important yet challenging task when applying machine learning methods to tabular data. 
It has been well recognized that the quality of features has a significant impact on the learning performance of tabular data~\citep{AFewUsefulThingsToKnowAboutMachineLearning}. The goal of feature generation is to transform the base features into more informative ones to better describe the data and enhance the learning performance. Figure \ref{fig:tabular_demo} demonstrates an example tabular data from~\citet{diabetes}, where the goal is to predict whether a patient will be readmitted to the hospital. In the example, we can transform the base features ``Age'' and ``\#\,Diagnose'' via  ``GroupByThenMean(Age, \#\,Diagnose)'', which calculates the mean number of diagnoses for each age group, and may be used to better characterize the age group and enhance the learning performance.
In practice, data scientists typically use their domain knowledge to find useful feature transformations in a trial-and-error manner, which requires tremendous human labor and expertise.

\begin{figure}
    \centering
    \includegraphics[width=0.48\textwidth]{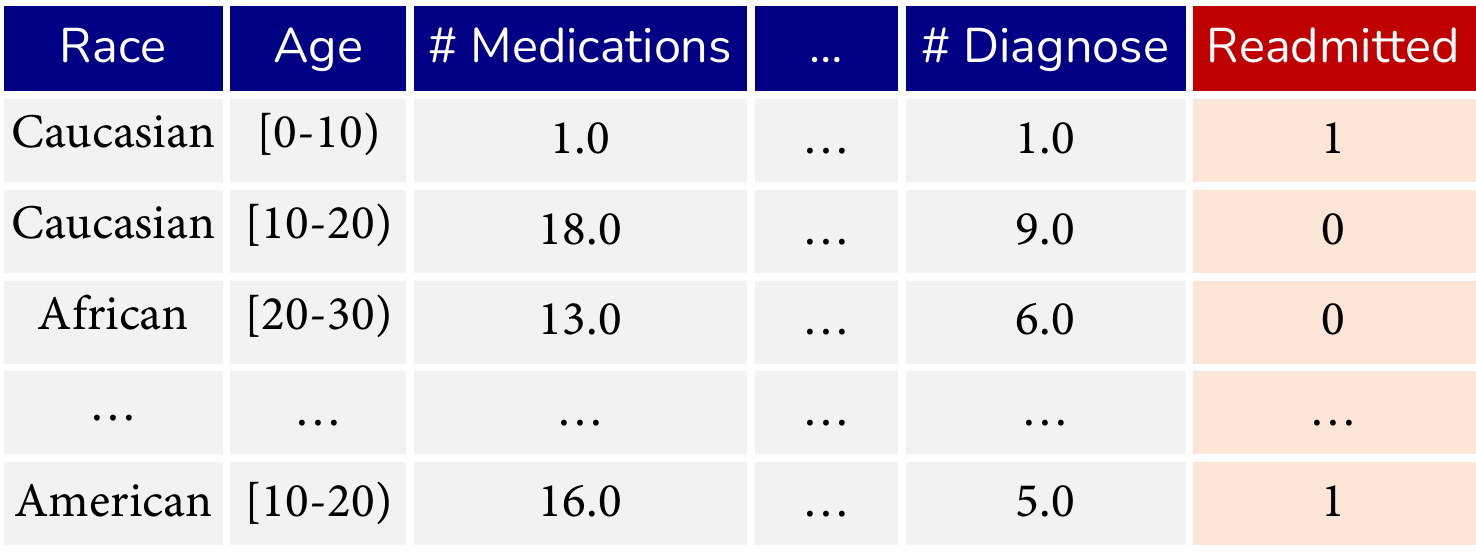}
    \caption{An example tabular data from \citet{diabetes}. The blue columns indicate the base features. The red (last) column is the prediction target.}
    \label{fig:tabular_demo}
    \vspace{-3mm}
\end{figure}

Since manual feature generation is time-consuming and requires case-by-case domain knowledge, automated feature generation emerges as an important topic in automated machine learning~\citep{autogluon,googleAI}. Expand-and-reduce is arguably the most prevalent framework in automated feature generation, in which we first create a pool of candidate features by expanding the base features and then eliminate the ineffective  candidate features~\citep{DSM,OneBM,autolearn,SAFE,ExploreKit}. There are two challenges in a typical expand-and-reduce practice. The first challenge is to efficiently and accurately evaluate the incremental performance of a new feature, i.e., how much performance improvement a new candidate feature can offer when added to the base feature set. A standard evaluation procedure involves including the new feature in the base feature set, retraining the machine learning model, and observing the change in the validation loss~\citep{ExploreKit}. However, retraining the model is usually time-consuming, making the standard evaluation procedure prohibitively expensive for large datasets. The second challenge is to calculate and evaluate the huge number of candidate features. There may be millions of candidate features for a dataset containing hundreds of features. Even with an efficient evaluation algorithm, computing all candidate features on the dataset is computationally costly and often impractical due to the enormous amount of memory required. 

This paper presents OpenFE, which addresses the two challenges by a feature boosting algorithm, and a two-stage pruning algorithm.
Regarding the first challenge, we argue that model retraining is not required for accurately evaluating the incremental performance of new features. Motivated by gradient boosting, we propose FeatureBoost, an efficient algorithm for evaluating the incremental performance of new features. FeatureBoost performs incremental training on top of the predictions yielded by the base feature set, which is substantially more efficient than retraining the model. For the second challenge, we propose a two-stage pruning algorithm to efficiently retrieve effective features from the vast pool of candidate features. Due to the fact that the effective features are usually sparse, the two-stage pruning algorithm performs feature pruning in a coarse-to-fine manner.
We validate OpenFE on various datasets and Kaggle competitions, where OpenFE outperforms existing baseline methods by a large margin. In two famous Kaggle competition with thousands of data science teams participating~\citep{IEEE,BNP}, the baseline model with features generated by OpenFE beats 99.3\% and 99.6\% data science teams. More importantly, the features generated by OpenFE result in comparable or even greater performance improvement than those provided by the competition's first-place team, demonstrating that OpenFE is competitive against machine learning experts in feature generation.

\begin{figure}[tb]
    \centering
    \includegraphics[width=.48\textwidth]{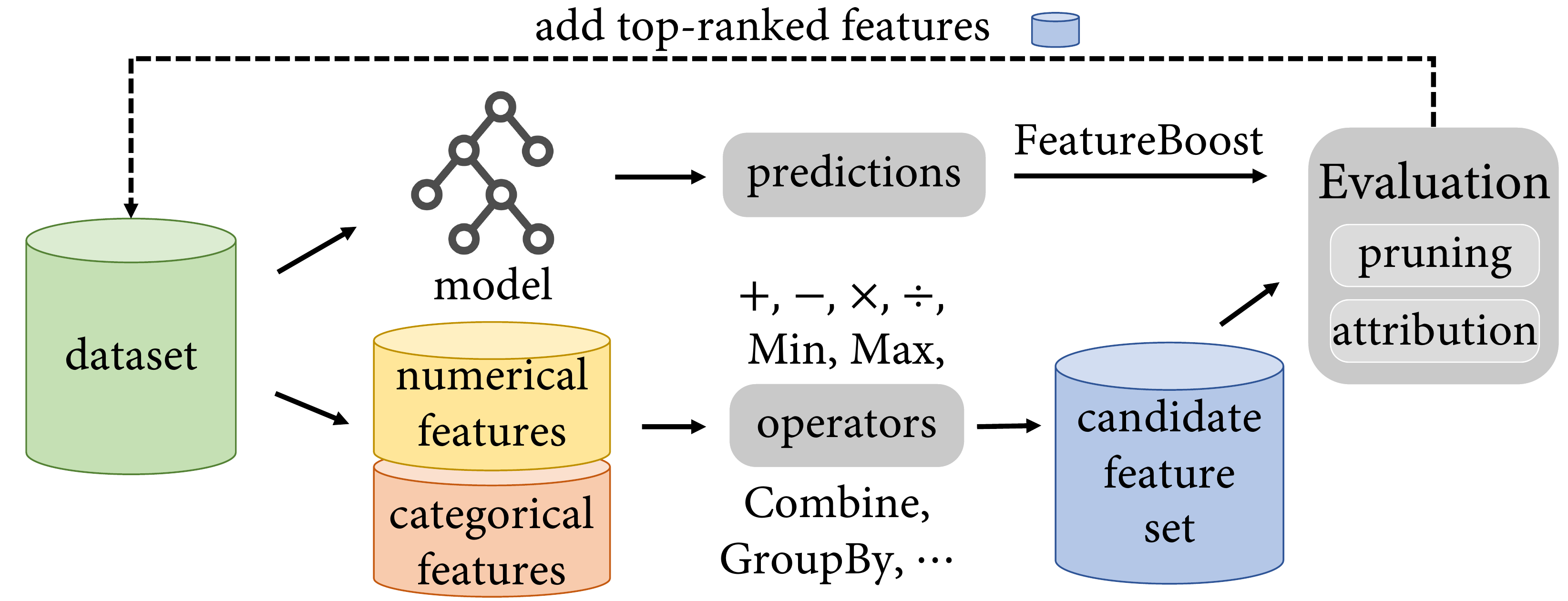}
    \caption{The overview of OpenFE.}
    \label{fig: overview}
    \vspace{-3mm}
\end{figure}

In addition to proposing a novel method, this paper intends to address two important problems that hinder the research process of automated feature generation. The first problem is a lack of fair comparisons of existing methods. The majority of existing methods do not open-source their codes, and some previous studies evaluate their methods without using a held-out test set~\citep{NFS,DBLP:conf/automl/DIFER,FETCH}. We reproduce most existing methods and perform a large-scale benchmarking on existing methods to facilitate fair comparisons in future research. The second problem is the lack of evidence regarding the necessity of feature generation in the era of deep learning. 
In recent years, a variety of deep neural networks (DNNs) have been developed for modeling tabular data~\citep{Tabnet,Revisiting}, and several of them have demonstrated their efficiency in feature interaction learning~\citep{Autoint,DCNV2}. We conduct a large-scale empirical study to demonstrate that the generated features can further enhance the learning performance of several state-of-the-art DNNs. In addition to the empirical results, we provide a theoretical justification of our feature generation procedure by presenting a simple yet representative transductive learning setting in which feature generation can bring benefit to the learning algorithm provably.
We summarize the contributions of our paper as follows:

\begin{itemize}
\item We propose OpenFE, a novel automated feature generation method that can efficiently and accurately identify useful new features to enhance learning performance. OpenFE achieves both efficiency and accuracy with a feature boosting method and a two-stage pruning algorithm.
\item Extensive experiments show that OpenFE outperforms baselines on $10$ benchmark datasets by a large margin. More importantly, the case study on two Kaggle competitions demonstrates that OpenFE can compete in feature generation against human experts.
\item We present a theoretical model to show that the model using the generated features can be statistically more advantageous than that using only the base features.
\item We facilitate future research by reproducing existing methods and performing a large-scale benchmarking on existing methods. The codes and datasets are available at \url{https://github.com/IIIS-Li-Group/OpenFE}
\end{itemize}


\section{Problem Definition}
For a given training dataset $\mathcal{D}$, we split it into a sub-training set $\mathcal{D}_{tr}$ and a validation set $\mathcal{D}_{vld}$. Assume $\mathcal{D}$ consists of a feature set $\mathcal{T}+\mathcal{S}$, where $\mathcal{T}$ is the base feature set and $\mathcal{S}$ is the generated feature set. We use a learning algorithm $\mathcal{L}$ to learn a model $\mathcal{L}(\mathcal{D}_{tr},\mathcal{T}+\mathcal{S})$, and compute the evaluation metric $\mathcal{E}(\mathcal{L}(\mathcal{D}_{tr},\mathcal{T}+\mathcal{S}),\mathcal{D}_{vld},\mathcal{T}+\mathcal{S})$ to measure the model performance, with a larger value indicating better performance. Formally, the feature generation problem is:
\begin{equation}
\max_{\mathcal{S}\subseteq A(\mathcal{T})} \mathcal{E}(\mathcal{L}(\mathcal{D}_{tr},\mathcal{T}+\mathcal{S}),\mathcal{D}_{vld},\mathcal{T}+\mathcal{S}),
\end{equation}
where $A(\mathcal{T})$ is the set of all possible candidate features generated from the base feature set. The goal of feature generation is to find a feature set $\mathcal{S}$ from $A(\mathcal{T})$ that maximizes the evaluation metric. 

\section{OpenFE}
\label{sec: overview}

\begin{algorithm}[tb]
    \caption{\;\texttt{OpenFE}}
    \label{alg: OpenFE} 
\begin{algorithmic}
   \STATE {\bfseries Input:} $\mathcal{D}$: dataset, $\mathcal{T}$: feature set, $\mathcal{O}$: operators
   \STATE {\bfseries Output:} new feature set
   \STATE Initialize order $= 1$.
   \WHILE{order $<$ predefined max order}
       \STATE $\triangleright$ Expansion step
       \STATE Initialize $A(\mathcal{T})$ by applying $\mathcal{O}$ on $\mathcal{T}$.
        \STATE $\triangleright$ Reduction step
        \STATE $\bm{\hat{y}} \leftarrow$ generate predictions with $\mathcal{T}$ on $\mathcal{D}$.
        \STATE $A(\mathcal{T})=\texttt{SuccessivePruning}(A(\mathcal{T}),\mathcal{D},\bm{\hat{y}}$).
        \STATE $A(\mathcal{T})=\texttt{FeatureAttribution}(A(\mathcal{T}),\mathcal{D},\bm{\hat{y}}$).
        \STATE $\mathcal{T} \leftarrow \mathcal{T} + \text{Top\_k}(A(\mathcal{T}))$.
        \STATE order $\leftarrow$ order + 1.
   \ENDWHILE
   \STATE \textbf{return} $\mathcal{T}$.
\end{algorithmic}
\end{algorithm}

OpenFE follows the expand-and-reduce framework~\citep{DSM,OneBM,autolearn,SAFE,ExploreKit} for automated feature generation, in which we first expand the candidate feature space, and then reduce the space by removing ineffective features. The procedure of OpenFE is presented in Algorithm~\ref{alg: OpenFE} and Figure \ref{fig: overview}. In the following we describe the expansion step and the reduction step in turn.

\subsection{The Expansion Step}

\begin{table}[bt]
\centering
\small
\caption{Examples of operators and their effects on features in Figure~\ref{fig:tabular_demo}. All examples follow the Value$_\mathrm{\,Name}$ format.}
\resizebox{1.0\linewidth}{!}{
\begin{tabular}{ccc}\toprule
Type & Operator & Example \\ \midrule
Unary &  $\mathrm{Freq}$ & African$_\mathrm{\,Race}$ $\rightarrow 0.21$ \\
Unary &  $\mathrm{Log}$ & ${1.0}_\mathrm{\,\# Diagnose}$ $\rightarrow 0.0$  \\
Binary & $\mathrm{Min}$ & ($18.0_\mathrm{\,\# Diagnose},9.0_\mathrm{\,\# Medications}$) $\rightarrow 9.0$ \\
Binary & $\div$ & ($16.0_\mathrm{\,\# Diagnose},5.0_\mathrm{\,\# Medications}$) $\rightarrow 3.2$\\
\bottomrule
\end{tabular}
}
\label{tab: op_example}
\end{table}

In the expansion step, we transform the base features by the operators to create a pool of candidate features.
The operators are critical in the expansion step, as they determine the search space for new features. The operators are classified into unary operators (such as $\log$, $\mathrm{sigmoid}$, $\mathrm{square}$), and binary operators, (such as $\times$, $\div$, $\min$, $\max$, $\mathrm{GroupByThenMean}$) according to the number of features they operate on. We present several examples that use the operators to transform base features in Table \ref{tab: op_example}. The details of all the operators are presented in Appendix \ref{app: sec: operators}. The candidate feature space is expanded via enumeration. For each operator, we iterate over all the base features and transform the base features into new features by the operator. For instance, for the operator ``$\div$'', we iterate over all possible numerical feature pairs $\{(\tau_1,\tau_2)\}$ to generate a list of candidate features $\{\tau_1\div\tau_2\}$. The candidate feature space contains all the first-order transformations of base features, where each transformation uses one operator. Given a dataset with $m$ features, the number of candidate features is $O(dm^2)$, where $d$ is the number of binary operators. There may be millions of candidate features for a dataset containing hundreds of features.



\subsection{The Reduction Step}
After the expansion step, we obtain a vast pool of candidate features. Therefore, the reduction framework aims at efficiently and accurately identifying the effective features from the vast pool of candidate features. To this end, first we propose FeatureBoost, an efficient algorithm to evaluate the effectiveness of new features. Although FeatureBoost provides an efficient way to evaluate each candidate feature, it is still computationally expensive to compute and evaluate the huge number of candidate features. Therefore, we propose a two-stage pruning algorithm on top of FeatureBoost to further accelerate the reduction step. 


\subsubsection{FeatureBoost: Fast Incremental Performance Evaluation}
\label{sec: feature boosting}

\begin{algorithm}[t]
   \caption{\;\;\texttt{FeatureBoost}}
   \label{alg: feature_boosting}
\begin{algorithmic}
   \STATE {\bfseries Input:} $\mathcal{D}$: dataset, $\mathcal{T}'$: feature set, $\bm{\hat{y}}$: predictions on $\mathcal{T}$
   \STATE {\bfseries Output:} incremental performance of $\mathcal{T}'$
   \STATE Initialize $L(f)$ as the objective function of $f$ with $\mathcal{T}$.
   \STATE Initialize a new model $f'$.
   \STATE Optimize $L(f')=\sum_{i=1}^nl(y_i,{\hat{y}_i}+f'(\bm{x}_i[{\mathcal{T}'}]))$
   \STATE $\Delta \leftarrow L(f)-L(f')$
   \STATE \bf{return} $\Delta$
   
\end{algorithmic}
\end{algorithm}

One of the key challenges in automated feature generation is to accurately estimate the incremental performance of a new feature, i.e., how much performance improvement the new feature can offer when added to the base feature set. 
A standard evaluation procedure involves including the new feature in the base feature set, retraining the machine learning model, and observing changes in the objective function~\citep{ExploreKit}. In the following we formally describe the standard evaluation procedure.
Given a dataset of $n$ samples, we denote the dataset $\mathcal{D}=\{(\bm{x}_i[\mathcal{T}],y_i)\,|\,i=1,2,\cdots,n\}$, where $\bm{x}_i[{\mathcal{T}}]\!\in\!\mathbb{R}^{|\mathcal{T}|}$ denotes the $i$-th sample projected to the feature set $\mathcal{T}$.
Given a trained machine learning model $f$ using $\mathcal{T}$, we denote the predictions as $\hat{y}_i=f(\bm{x}_i[{\mathcal{T}}])$ and the objective function as $L(f)=\sum_{i=1}^nl(y_i,f(\bm{x}_i[{\mathcal{T}}]))$, where $l(\cdot,\cdot)$ is the loss function (e.g., the mean squared error).
When introducing a set of new features $\mathcal{T}'$, the standard procedure retrains a new model $f'$ with the full feature set $\mathcal{T}+\mathcal{T}'$, and computes the objective function $L(f')=\sum_{i=1}^nl(y_i,f'(\bm{x}_i[{\mathcal{T}+\mathcal{T}'}]))$.
The loss improvement $L(f)-L(f')$ indicates the incremental performance of $\mathcal{T}'$, which is larger the better.
However, retraining the model with the full feature set is usually time-consuming, making the standard evaluation procedure prohibitively expensive.

\begin{algorithm}[t]
   \caption{\;\;\texttt{SuccessivePruning}}
   \label{alg: successive}
\begin{algorithmic}
   \STATE {\bfseries Input:} $\mathcal{D}$: dataset, $\bm{\hat{y}}$: predictions on $\mathcal{T}$,\\ $A(\mathcal{T})$: candidate feature set, $q$: integer
   \STATE {\bfseries Output:} pruned new feature set
   \STATE Divide $\mathcal{D}$ equally into $2^q$ data blocks.
   \STATE $A_0(\mathcal{T}) \leftarrow A(\mathcal{T})$.
   \FOR{$i=0$ {\bfseries to} $q$}
       \STATE $\triangleright$ Create a subset $\mathcal{D}_i$ with $2^{i}$ randomly selected data blocks
        \FOR{new feature $\tau\in A_i(\mathcal{T})$}
            \STATE $\Delta_{\tau}=\texttt{FeatureBoost}(\mathcal{D}_i,\{\tau\},\bm{\hat{y}})$.
        \ENDFOR
        \STATE $A_{i}(\mathcal{T}) \leftarrow$ deduplicate $A_{i}(\mathcal{T})$.
        \STATE $A_{i+1}(\mathcal{T}) \leftarrow$ Take the top half of $A_{i}(\mathcal{T})$ based on $\Delta$.
   \ENDFOR
   \FOR{$\tau\in A_{q+1}(\mathcal{T})$}
       \IF{$\Delta_{\tau} \leq 0$}
            \STATE $A_{q+1}(\mathcal{T}) \leftarrow A_{q+1}(\mathcal{T}) \backslash \{{\tau}\}$.
        \ENDIF
    \ENDFOR
   \STATE \textbf{return} $A_{q+1}(\mathcal{T})$
\end{algorithmic}
\end{algorithm}

In this paper, we argue that it is unnecessary to retrain the model to estimate the incremental performance of new features.
In response, we present FeatureBoost, a fast incremental performance evaluation algorithm.
As illustrated in Algorithm~\ref{alg: feature_boosting}, the input to FeatureBoost comprises the dataset $\mathcal{D}$, the feature set $\mathcal{T}'$, and the prediction $\bm{\hat{y}}$.
Inspired by gradient boosting, FeatureBoost only trains a model with the new feature set $\mathcal{T}'$ to fit the residuals between $\bm{\hat{y}}$ and the target $\bm{y}$.
Formally, it optimizes the new model $f'$ to minimize $L(f')=\sum_{i=1}^nl(y_i,{\hat{y}_i}+f'(\bm{x}_i[{\mathcal{T}'}]))$.
Finally, $\Delta=L(f)-L(f')$ serves as an estimate of the incremental performance of $\mathcal{T}'$.


In general, the input feature set for FeatureBoost is the new features, which allows the algorithm to only train $f'$ on few features to be efficient.
This implementation is employed in the first stage of the pruning algorithm (introduced in the next subsection), which coarsely reduces the number of candidate features.
However, in the second stage, we need to consider the interaction effects between the new features and the original features, which is more accurate when evaluating the incremental performance of the new features.
In this case, the input $\mathcal{T}'$ is in fact the full feature set (i.e., $\mathcal{T}+\mathcal{T}'$) to allow the model to better rank the new features.
It is worth noting that even when the input becomes full set, FeatureBoost is still superior to the standard evaluation because it exploits $\bm{\hat{y}}$ to converge faster.

\subsubsection{A Two-stage Pruning Algorithm}
Although FeatureBoost provides an efficient way to estimate the incremental performance of candidate features, it is still computationally expensive to compute and evaluate the huge number of candidate features.
Therefore, we propose a two-stage pruning algorithm on top of FeatureBoost to efficiently eliminate redundant new features.
Motivated by the fact that effective features are usually sparse in the candidate feature space, we devise a two-stage pruning algorithm to perform feature pruning in a coarse-to-fine manner. Concretely, the first stage is to coarsely reduce the number of candidate features by examining the effectiveness of each feature alone, while the second stage takes into account the fine-grained interaction between features.

\begin{algorithm}[t]
    \caption{\;\;\texttt{FeatureAttribution}}
    \label{alg: attribution} 
\begin{algorithmic}
   \STATE {\bfseries Input:} $\mathcal{D}$: dataset, $\mathcal{T}$: feature set, $A'(\mathcal{T})$: candidate feature set, $\bm{\hat{y}}$: predictions on $\mathcal{T}$
   \STATE {\bfseries Output:} sorted $A'(\mathcal{T})$.
   \STATE $\Delta=\texttt{FeatureBoost}(\mathcal{D},\mathcal{T}+A'(\mathcal{T}),\bm{\hat{y}})$.
   \STATE Attribute importance scores to each feature in $A'(\mathcal{T})$ according to their contributions to the loss reduction $\Delta$.
  \STATE $A'(\mathcal{T}) \leftarrow$ sort $A'(\mathcal{T})$ based on their importance scores.
   \STATE \textbf{return} $A'(\mathcal{T})$
\end{algorithmic}
\end{algorithm}

\paragraph{Stage I: Successive Featurewise Pruning}
We propose a successive featurewise pruning algorithm to quickly reduce the number of candidate features (Algorithm \ref{alg: successive}). The algorithm is motivated by the successive halving algorithm in multi-armed bandit problems~\citep{even2006action,zhou2014optimal}, where successive halving dynamically allocates computing resources to promising arms. In our settings, first we split the dataset into $2^q$ data blocks, where each data block has $\lfloor \frac{n}{2^q} \rfloor$ samples and $q$ is a hyper-parameter. Then the algorithm proceeds iteratively to remove redundant candidate features. In the $i$-th iteration, first we create a subset $\mathcal{D}_i$ by randomly selecting $2^i$ data blocks. Then, for each candidate feature $\tau$, we compute $\Delta_{\tau}=\texttt{FeatureBoost}(\mathcal{D}_i,\{\tau\},\bm{\hat{y}})$ as the score of the candidate feature. In the end of each iteration, we keep the top half of candidate features and eliminate the rest according to the scores $\Delta$.
Finally, each candidate feature $\tau$ with a positive $\Delta_{\tau}$ can be returned.
Besides, if there are two features with exactly the same values, we remove one of them for deduplication.
For example, $\max(\tau_1,\tau_2)$ and $\max(\tau_1,\tau_3)$ are exactly the same when the minimum value of $\tau_1$ is larger than the maximum value of both $\tau_2$ and $\tau_3$. In this case, we only keep $\max(\tau_1,\tau_2)$ as the candidate feature.


\paragraph{Stage II: Feature Importance Attribution}
Unlike the coarse stage I, stage II considers the fine-grained interaction effects between the candidate features and the base features (Algorithm~\ref{alg: attribution}).
Let the candidate features after pruning be $A'(\mathcal{T})$. We use candidate features $A'(\mathcal{T})$ and base features $\mathcal{T}$ together as the inputs to FeatureBoost. $\Delta=\texttt{FeatureBoost}(\mathcal{D},\mathcal{T}+A'(\mathcal{T}),\bm{\hat{y}})$ evaluates the incremental performance of the remaining candidate features $A'(\mathcal{T})$, which considers the interaction effects between $\mathcal{T}$ and $A'(\mathcal{T})$.
We are interested in the contribution of each candidate feature in the loss reduction $\Delta$. 
Feature importance attribution methods can attribute the importance scores to each feature~\citep{SHAP}. Popular methods include mean decrease in impurity (MDI)~\citep{random_forests}, permutation feature importance (PFI)~\citep{random_forests}, and SHAP~\citep{SHAP}. We rank the candidate features according to their importance scores. Finally, we only select the top-ranked candidate features to improve the generalization of our algorithm.

\subsection{Implementation}
In OpenFE, we use gradient boosting decision trees (GBDT)~\citep{AGradientBoostingMachine} to model tabular data for two reasons: 1) GBDT is usually the best performing model on tabular data where features are individually meaningful \citep{Revisiting, deep-tabular-survey}. 2) GBDT can automatically handle missing values and categorical features, which is convenient for automation~\citep{lightgbm}. We use the popular LightGBM implementation \citep{lightgbm}. We use MDI for feature importance attribution. We compare MDI, PFI, and SHAP in Appendix \ref{app: sec: feature importance}. Even though the feature generation method relies on GBDT, the generated features can also enhance the learning performance of a variety of DNNs (see Section \ref{sec: nn}).

\section{Theoretical Advantage of Feature Generation}

In this section, we study the advantage of feature generation from a theoretical perspective. We present a simple yet representative setting in which the test loss of empirical risk minimization using both base features and generated features converges to zero provably as the number of training samples increases, while the test loss for any learning model using only base features is at least a positive constant. In particular, we present a transductive learning setting, which can capture important characteristics of many datasets one may encounter frequently in data science applications, e.g., the IEEE-CIS Fraud Detection dataset (we also conduct experiments on this dataset in Section~\ref{sec: kaggle}). Due to space limit, a formal and detailed description of the model can be found in Appendix~\ref{app: sec: theory}. We briefly introduce the high-level idea here.

Many tabular datasets contain both categorical and numerical attributes (i.e.,
features). A categorical feature partitions the dataset into groups
(each associated with a distinct category). For a data point $(X,Y)$.
the target $Y$ is correlated with not only the feature $X$, but also 
certain statistics of the group containing $(X,Y)$.
Datasets with such characteristics are abundant in data science applications. 
As a concrete example, suppose each training data point
records a transaction,
and one special categorical feature is user id (each user may have many transactions
in the table. All transactions of the same user form a group). The target $Y$ we want to predict
about the transaction (e.g., probability of fraudulence)
may depend on not only the features of this particular transaction, but also 
some statistics of this user (e.g., average size of his/her transactions). Hence, one can see that operations 
such as $\mathrm{GroupByThenMean}$ (group by the user id) can provide statistical information about the user by aggregating the information 
from all data points associated to this user. 

We present a theoretical data model, which is a two-phase data generation model, to capture the above characteristics.
Under fairly standard learning theoretic assumptions (i.e., bounded Rademacher complexity), we prove that empirical risk minimization augmented with feature generation
(such as the $\mathrm{GroupByThenMean}$ operation) can achieve vanishing test loss as the sample size and group size increase.

\begin{theorem}
(informal) Assume the data set is generated according to 
the two-phase process described in Appendix~\ref{app: sec: theory}.
Denote the number of groups in the training set and test set
by $\ktrain$ and $\ktest$ respectively, and the number of data points
in each group by $\nmin$.
There is a feature generation function $H$, such that 
the test loss of the empirical risk minimizer $\hat{f}$ can be bounded by 
\begin{align*}
    L_{\Dtest}(H, \hat{f}) \leq &
O\biggl( \mathrm{Rad}_{\ktrain}(\mathcal{F}) + 
\sqrt{  \ln(4\delta^{-1})/\ktrain} \\
 & + \sqrt{d  \ln (4d(\ktrain + \ktest)\delta^{-1}) / h}
\biggr)
\end{align*}
with probability at least $1 - \delta$.
In particular, 
assuming the Rademacher complexity $\mathrm{Rad}_{\ktrain}(\mathcal{F})\rightarrow 0$
as $\ktrain\rightarrow\infty$, the test loss approaches to $0$ 
when $\ktrain, \nmin \rightarrow \infty$. 
\end{theorem}

On the other hand, if we do not use any feature generation,
we prove that any predictor $f'$ (no matter how complicated $f'$ is) 
incurs a non-vanishing constant test loss.

\begin{table*}[h]

\centering
\caption{Properties of datasets used in our experiments. Notation: ``RMSE" $\sim$ root-mean-square error, ``AUC'' $\sim$ area-under-curve, ``Acc." $\sim$ accuracy. }
\resizebox{0.8\textwidth}{!}{
    \begin{tabular}{lcccccccccc}
    \toprule
    \multirow{2}{*}{\textbf{Dataset}} & \multicolumn{3}{c}{\textbf{Regression}} & \multicolumn{7}{c}{\textbf{Classification}} \\
    \cmidrule(lr){2-4}
    \cmidrule(lr){5-11}
    {} & \textbf{CA} & \textbf{MI} & \textbf{ME} & \textbf{TE} & \textbf{BR} & \textbf{DI} & \textbf{NO} & \textbf{VE} & \textbf{JA} & \textbf{CO} \\
    \midrule
    \# samples (k) & 20.6 & 1200 & 163 & 51.0 & 900 & 102 & 34.4 & 98.5 & 83.7 & 581 \\
    \# numerical features & 7 & 111 & 5 & 21 & 31 & 3 & 34 & 100 & 54 & 9 \\
    \# categorical features & 0 & 0 & 6 & 22 & 0 & 34 & 29 & 0 & 0 & 0 \\
    \# ordinal features & 1 & 25 & 0 & 14 & 27 & 10 & 55 & 0 & 0 & 45 \\
    \# classes & -- & -- & -- & 2 & 2 & 2 & 2 & 2 & 4 & 7 \\
    metric & RMSE & RMSE & RMSE & AUC & AUC & AUC & AUC & AUC & Acc. & Acc. \\
    \bottomrule
    \end{tabular}
}
\label{tab: datasets}
\end{table*}

\begin{table*}[h]
\centering
\small
\caption{Comparisons between OpenFE and baseline methods. The results that demonstrate a significant improvement over others are highlighted in \textbf{bold}. We repeat each experiment 10 times and apply Welch’s t-test with unequal variance and a p-value of 0.05 to assess the significance. \xmark \ denotes a failure due to exceeding the runtime limits (24 hours). -- means that the method cannot run on multi-classification and regression datasets.}
\renewcommand\arraystretch{1.05}
\resizebox{1.0\linewidth}{!}{
\begin{tabular}{lcccccccccc}
\toprule
Method & CA $\downarrow$ & MI $\downarrow$ & ME $\downarrow$ & TE $\uparrow$ & BR $\uparrow$ & DI $\uparrow$ & NO $\uparrow$ & VE $\uparrow$ & JA $\uparrow$ & CO $\uparrow$ \\ \midrule
Base & 0.432$_{\pm0.004}$ & 0.744$_{\pm0.000}$ & 1128.4$_{\pm1.639}$ & 0.671$_{\pm0.002}$ & 0.756$_{\pm0.004}$ & 0.731$_{\pm0.001}$ & 0.996$_{\pm0.000}$ & 0.925$_{\pm0.000}$ & 0.721$_{\pm0.002}$ & 0.969$_{\pm0.001}$ \\ 
NN & 0.479$_{\pm0.001}$ & 0.750$_{\pm0.000}$ & 1413.7$_{\pm2.425}$ & 0.661$_{\pm0.002}$ & 0.748$_{\pm0.004}$ & 0.717$_{\pm0.001}$ & 0.992$_{\pm0.000}$ & 0.924$_{\pm0.001}$ & 0.720$_{\pm0.001}$ & 0.966$_{\pm0.000}$ \\ 
FCTree & 0.432$_{\pm0.003}$ & 0.744$_{\pm0.000}$ & 1088.9$_{\pm1.072}$ & 0.670$_{\pm0.001}$ & 0.750$_{\pm0.004}$ & 0.731$_{\pm0.001}$ & 0.996$_{\pm0.000}$ & 0.926$_{\pm0.000}$ & 0.719$_{\pm0.001}$ & 0.971$_{\pm0.001}$ \\ 
SAFE & -- & -- & -- & 0.673$_{\pm0.001}$ & 0.750$_{\pm0.004}$ & 0.730$_{\pm0.002}$ & 0.996$_{\pm0.000}$ & 0.925$_{\pm0.001}$ & -- & -- \\ 
AutoFeat & 0.444$_{\pm0.002}$ & 0.744$_{\pm0.000}$ & 1172.2$_{\pm2.235}$ & 0.672$_{\pm0.002}$ & 0.750$_{\pm0.005}$ & 0.732$_{\pm0.001}$ & 0.996$_{\pm0.000}$ & 0.925$_{\pm0.000}$ & 0.721$_{\pm0.001}$ & 0.968$_{\pm0.001}$ \\ 
AutoCross & -- & -- & -- & 0.651$_{\pm0.002}$ & 0.765$_{\pm0.001}$ & 0.732$_{\pm0.001}$ & 0.993$_{\pm0.000}$ & 0.921$_{\pm0.000}$ & -- & -- \\ 
FETCH & 0.430$_{\pm0.002}$ & \xmark & 1130.6$_{\pm1.045}$ & 0.673$_{\pm0.002}$ & \xmark & 0.731$_{\pm0.002}$ & 0.996$_{\pm0.000}$ & 0.927$_{\pm0.000}$ & 0.720$_{\pm0.001}$ & \xmark \\ \midrule
\textbf{OpenFE} & \textbf{0.421$_{\pm0.002}$} & \textbf{0.738$_{\pm0.000}$} & \textbf{982.0$_{\pm1.745}$} & \textbf{0.680$_{\pm0.002}$} & \textbf{0.786$_{\pm0.002}$} & \textbf{0.888$_{\pm0.002}$} & \textbf{0.997$_{\pm0.000}$} & 0.928$_{\pm0.000}$ & \textbf{0.729$_{\pm0.001}$} & \textbf{0.974$_{\pm0.000}$} \\ \bottomrule
\end{tabular}
}
\label{tab: baselines_comparison}
\end{table*}

\begin{theorem}
(informal)
In case that we do not use any feature generation, there exists a problem instance such that, no matter how large $\ktrain, \ktest$, and $h$ are, for any predictor $f': \mathcal{X}\rightarrow \mathcal{Y}$, the test loss  
$
L_{\Dtest}(f') \geq \frac{3}{64}.
$
\end{theorem}

\section{Experiments}

\subsection{Datasets and Evaluation Metrics}
We collect a diverse set of ten public datasets. Most datasets are frequently used in previous studies~\citep{Revisiting,num_embedding,why-do-tree,diabetes,vehicle,nomao}. Each dataset has exactly one train-validation-test split, and all methods use the same split. The datasets we collect include: California Housing (CA), Microsoft (MI), Medical (ME), Diabetes (DI), Nomao (NO), Vehicle (VE), Broken Machine (BR), Telecom (TE), Jannis (JA), Covertype (CO). We summarize the dataset properties in Table \ref{tab: datasets}. 
See more detailed descriptions in the appendix.

\subsection{Baseline Methods for Comparisons}
The baseline methods for comparisons include: \textbf{Base}, the base feature set without feature generation, \textbf{FCTree}~\citep{FCTree}, \textbf{SAFE}~\citep{SAFE}, \textbf{AutoFeat}~\citep{AutoFeat}, \textbf{AutoCross}~\citep{AutoCross}, and \textbf{FETCH}~\citep{FETCH}. We provide another baseline that uses the last hidden layer of DCN-V2 as new features. DCN-V2~\citep{DCNV2} is a DNN architecture developed for modeling tabular data and claims to be able to capture feature interactions automatically. We denote this baseline as \textbf{NN}. Most existing automated feature engineering methods do not open-source their codes. We reproduce FCTree, SAFE, and AutoCross according to the descriptions in these papers. We validate our reproduction by comparing the reproduced results with those in the corresponding papers (see Appendix \ref{app: sec: reproduce}). Some other learning-based methods, such as LFE~\citep{LFE}, ExploreKit~\citep{ExploreKit} (meta learning) and NFS~\citep{NFS}, TransGraph~\citep{RL} (reinforcement learning), lack critical details for reproduction. For example, the choice of training datasets is crucial for generalization in learning-based methods, however previous research did not describe which datasets were utilized for training~\citep{LFE,RL}. We run OpenFE on their datasets and compare our results with the numbers reported in their papers in Appendix \ref{app: sec: comparison_with_other_baselines}.

\subsection{Comparison Results}
We compare OpenFE with other baseline methods. We use LightGBM~\citep{lightgbm} as the learning algorithms to evaluate the effectiveness of the new features generated by different methods. Hyperparameter tuning follows a standard benchmarking study~\citep{Revisiting}, and we tune the hyperparameters using the base feature set. For each dataset, we generate first-order features and include the same number of new features for different methods for a fair comparison. We discuss generating high-order features in Appendix \ref{app: sec: high-order}. We present the comparison results in Table \ref{tab: baselines_comparison}. We can observe that OpenFE has clear advantages over baseline methods. When baseline methods fail to generate effective new features, they include uninformative features in the base feature set, which brings additional noise and does harm to the generalization of the learning model.

\begin{table}[t]
\centering
\caption{The effect of OpenFE on a variety of DNNs. We report the results on the test set over 10 random seeds (see Appendix \ref{app: sec: standard deviations} for standard deviations). Statistically significant improvements are marked in \textbf{bold}.}
\resizebox{\linewidth}{!}{%
    \begin{tabular}{clccccccc}
    \toprule
    \rule{0pt}{10pt} \textbf{Model} & \textbf{Features} & \textbf{CA} \textdownarrow & \textbf{MI} \textdownarrow & \textbf{NO} \textuparrow & \textbf{VE} \textuparrow & \textbf{JA} \textuparrow 
    \\
    \midrule
    \multirow{2}{*}{AutoInt}
    & Base
    & $0.474$
    & $0.750$
    & $0.993$
    & $0.927$ 
    & $0.720$
    \\
    & OpenFE
    & $\mathbf{0.462}$ 
    & $\mathbf{0.745}$ 
    & $\mathbf{0.994}$ 
    & $\mathbf{0.929}$  
    & $\mathbf{0.724}$  
    \\
    \midrule
    \multirow{2}{*}{DCN-V2}
    & Base
    & $0.483$
    & $0.749$
    & $0.993$
    & $0.924$
    & $0.716$
    \\
    & OpenFE
    & $\mathbf{0.472}$ 
    & $\mathbf{0.743}$ 
    & $\mathbf{0.994}$ 
    & $\mathbf{0.926}$  
    & $0.716$  
    \\
    \midrule
    \multirow{2}{*}{TabNet} 
    & Base
    & $0.510$
    & $0.751$
    & $0.993$
    & $0.924$
    & $0.724$
    \\
    & OpenFE
    & $\mathbf{0.501}$ 
    & $\mathbf{0.746}$ 
    & $\mathbf{0.995}$ 
    & $\mathbf{0.926}$  
    & $0.724$  
    \\
    \midrule
    \multirow{2}{*}{NODE}  
    & Base
    & $0.464$
    & $0.745$
    & $0.994$
    & $0.927$
    & $0.727$
    \\
    & OpenFE
    & $\mathbf{0.457}$ 
    & $\mathbf{0.740}$ 
    & $\mathbf{0.995}$ 
    & $\mathbf{0.929}$  
    & $\mathbf{0.731}$  
    \\
    \midrule
    \multirow{2}{*}{FT-T} 
    & Base
    & $0.459$
    & $0.746$
    & $0.993$
    & $0.927$
    & $0.733$
    \\
    & OpenFE
    & $\mathbf{0.453}$ 
    & $\mathbf{0.741}$ 
    & $\mathbf{0.995}$ 
    & $\mathbf{0.929}$  
    & $\mathbf{0.738}$  
    \\
    \bottomrule
    \end{tabular}
}
\renewcommand{\arraystretch}{1}
\label{tab: nn_results}
\end{table}

\subsection{Feature Generation for Neural Networks}
\label{sec: nn}
We show that the new features generated by OpenFE can greatly improve the performance of a variety of neural networks designed specifically for tabular data. The models include: \textbf{AutoInt}~\citep{Autoint}, \textbf{DCN-V2}~\citep{DCNV2}, \textbf{NODE}~\citep{NODE}, \textbf{TabNet}~\citep{Tabnet}, \textbf{FT-Transformer}~\citep{Revisiting}. We follow the same implementations and hyperparameter tuning in~\citet{Revisiting}. We present the results in Table \ref{tab: nn_results}. The features generated by OpenFE greatly enhance the performance of different models in most cases.

\subsection{Compare OpenFE with Kaggle Experts}
\label{sec: kaggle}

The first Kaggle competition is IEEE-CIS Fraud Detection~\citep{IEEE}, where the goal is to predict whether an online transaction is fraudulent. This competition is one of the largest and most competitive tabular data competitions on Kaggle, with 6,351 data science teams participating. The competition's first place team made public the features they generated after the competition ended~\citep{IEEE-first-place}, which we refer to as Expert. A baseline model of XGBoost~\citep{xgboost} without feature generation ranks at 2286 among 6351 teams on the private leaderboard. The baseline model with features generated by the team ranks at 76/6351. The same baseline model with features generated by OpenFE ranks at 42/6351, which outperforms the features generated by the first-place team. We present the results in Table \ref{tab: kaggle}. The first-place team only shared one of the XGBoost models they used. The actual model they used to achieve the first rank is an ensemble of many learning models (including LightGBM, CatBoost~\citep{catboost}, and XGBoost). But the team withheld the details of these models~\citep{IEEE-first-place}, which is the reason their reproduced result does not rank first. 

The second competition is BNP Paribas Cardif Claims Management~\citep{BNP}, where the goal is to predict whether an insurance claim should be approved. The competition's $8$-th place team made public their generated features after the competition ended~\citep{BNP-8-place} (the winners with higher rankings did not share their codes). We evaluate the performance of OpenFE in a similar way. A baseline model using Catboost~\citep{catboost} without feature generation ranks at 31 among 2920 teams. The baseline model with features generated by experts ranks at 12/2920, while the baseline model with features generated by OpenFE also ranks at 12/2920.

\begin{table}[t]
\centering
\caption{Results of OpenFE and Expert (feature generation by the $1$-st place team in IEEE and the $8$-th place team in BNP) in two Kaggle competitions. Notation: pub.~$\sim$ public score, pri.~$\sim$ private score.}
\renewcommand\arraystretch{1.2}
\resizebox{0.9\linewidth}{!}{
\begin{tabular}{llcccc}
\toprule
\textbf{Feature} & \textbf{Order}& \multicolumn{2}{c}{IEEE \textuparrow}   
                          & \multicolumn{2}{c}{BNP \textdownarrow} \\
                          \cmidrule(lr){3-4} \cmidrule(lr){5-6}
&       
& pub. 
& pri.
& pub.    
& pri.    
\\ \midrule
Base                      
& --           
& $0.946$     
& $0.918$     
& $0.438$                   
& $0.436$         
\\ \midrule
\multirow{2}{*}{Expert} 
& first 
& $0.960$       
& $0.933$      
& $0.433$       
& $0.431$  
\\
& all  
& $0.960$      
& $0.932$      
& $\bf{0.432}$                       
& $\bf{0.430}$           
\\ \midrule
\multirow{2}{*}{OpenFE}   
& first 
& $\bf{0.962}$       
& $\bf{0.936}$      
& $0.435$   
& $\bf{0.432}$         
\\
& all  
& $\bf{0.962}$     
& $\bf{0.936}$      
& $0.432$                   
& $\bf{0.430}$                
\\ \bottomrule
\end{tabular}
}
\label{tab: kaggle}
\end{table}

\begin{table}[ht]
\centering
\caption{Results of the ablation study.}
\renewcommand\arraystretch{1.2}
\resizebox{\linewidth}{!}{%
\begin{tabular}{lcccccc}
    \toprule
    {} & \textbf{CA} \textdownarrow & \textbf{MI} \textdownarrow & \textbf{DI} \textuparrow & \textbf{NO} \textuparrow & \textbf{VE} \textuparrow & \textbf{CO} \textuparrow \\
    \midrule
    OpenFE         
    & $0.421$ 
    & $0.738$ 
    & $0.888$ 
    & $0.997$ 
    & $0.928$ 
    & $0.974$ \\
    OpenFE-MI
    & $0.428$ 
    & $0.744$ 
    & $0.887$ 
    & $0.996$ 
    & $0.926$ 
    & $0.971$ \\
    \textit{w.o.} Successive 
    & $0.428$ 
    & $0.744$ 
    & $0.885$ 
    & $0.995$ 
    & $0.927$ 
    & $0.965$ \\
    \bottomrule
\end{tabular}
}
\label{tab: ablation study}
\end{table}

\subsection{Analysis and Discussion}
\label{sec: analysis and discussion}
\paragraph{Ablation Study}
\label{sec: ablation}
We conduct an ablation study to justify the design choices of OpenFE. We name different variants of OpenFE as follows: 1) \textbf{OpenFE-MI}. Instead of using FeatureBoost to evaluate the effectiveness of new features, we use mutual information between the feature and the target as the scoring criterion. 2) \textbf{\textit{w.o.} Successive}. OpenFE without the successive featurewise pruning. We subsample the data so that generating all the features can fit in the memory. We present the results in Table \ref{tab: ablation study}. One can see from the results that: 1) Mutual information is worse than FeatureBoost in evaluating the effectiveness of new features. 2) Directly subsampling the data usually hurts the performance.

\input{pdf/runtime.tex}

\paragraph{Runtime Comparisons}
We compare the runtime of different methods on the TE, DI, NO, VE datasets. The results are presented in Figure \ref{fig: runtime}. One can see that OpenFE is significantly more efficient than AutoFeat, AutoCross, and FETCH. OpenFE only runs a few minutes to generate features on most medium-scale datasets. We provide the runtime on different datasets in Appendix \ref{app: sec: running time}. Complexity analysis is provided in Appendix \ref{app: sec: complexity}.

\paragraph{Discussion of High-order Features}
During the experiments, we find that generating second-order features is hardly effective for most benchmarking datasets. Table \ref{tab: kaggle} presents the Kaggle results using only first-order features and both first-order and high-order features (all) separately. Generating high-order features does not seem to be helpful for Expert or OpenFE in the IEEE competition. In the BNP competition, generating high-order features provides a slight improvement in the test score. First-order features are usually more important than high-order features in feature generation. See more detailed discussions in Appendix \ref{app: sec: high-order}.

\paragraph{Is Feature Generation Effective for All Datasets?} In order to answer this question, we conduct a large-scale empirical study using the OpenML CC18 benchmark~\citep{DBLP:conf/nips/OpenML-CC18} with 68 dataset (see results in Table \ref{app: tab: CC18}). We found that OpenFE did not result in significant improvements for 19 of those datasets, and neither did other feature generation methods. It is possible that some datasets do not necessitate feature generation. While we do not fully understand which datasets benefit most from feature generation, we would like to highlight that OpenFE is highly efficient, allowing researchers to quickly obtain preliminary results and assess the impact of feature generation for their dataset.

\section{Related Work}
Expand-and-reduce is arguably the most popular framework in automated feature generation. Most existing feature generation methods employ statistical methods to identify and remove redundant features. For example, Data Science Machine (DSM)~\citep{DSM} uses the f-value and One Button Machine (OneBM)~\citep{OneBM} uses the Chi-square test between the feature and the target to remove redundant features. DSM and OneBM focus on feature generation in relational databases. FCTree~\citep{FCTree} uses information gain and SAFE~\citep{SAFE} uses information value to exclude uninformative features. AutoCross~\citep{AutoCross} and AutoFeat~\citep{AutoFeat} uses the improvement on a linear regression model to evaluate a new feature. LFE~\citep{LFE} and ExploreKit~\citep{ExploreKit} design meta features based on statistical methods and use a meta learning approach to determine the effectiveness of new features.

Although statistical methods are widely employed, there are many studies showing that statistically significant features do not always translate into good predictors~\citep{significant, civil, financial,zhang2021chartnavigator}. The effectiveness of a new feature may be encompassed by the base feature set, even if the new feature is significantly correlated with the target. Another line of works employs the standard approach to evaluate the incremental performance of new features. Neural Feature Search~\citep{NFS} utilizes a recurrent neural network controller to search for new features. \citet{RL} uses reinforcement learning to traverse a transformation graph for high-order features. FETCH~\citep{FETCH} built a feature engineering pipeline based on the Markov decision process.

\section{Conclusion \& Limitations}
We propose OpenFE, a powerful automated feature generation tool with FeatureBoost and a two-stage pruning algorithm. Extensive experiments show that OpenFE achieves SOTA on ten benchmark datasets and is competitive against human experts in feature generation. We provide theoretical evidence that feature generation is crucial in modeling tabular data. We open-source the codes and datasets.

OpenFE currently has some limitations: 1) OpenFE cannot handle time series data. For example, we may leverage statistical ideas developed in time series literature, and the new features should satisfy the time series constraint (they should be computed from the past data). OpenFE cannot handle such constraints. 2)  Even though OpenFE is highly efficient and can handle large-scale datasets with millions of samples and hundreds of features, in some cases where the dataset is exceptionally large or computing resources are limited, one could perform feature selection or use a fraction of the base features to generate candidate features.

\section{Acknowledgement}
Tianping Zhang, Zheyu Zhang, Zhiyuan Fan, and Jian Li are supported in part by the National Natural Science Foundation of China Grant 62161146004, Turing AI Institute of Nanjing and Xi'an Institute for Interdisciplinary Information Core Technology.

\begin{table*}
    \centering
    \caption{Additional comparison results on 49 datasets from OpenML CC18 benchmark where feature generation can yield a significant improvement. The metric is accuracy for all datasets. OpenFE outperforms all baseline methods significantly on 40 out of these 49 datasets and provides an average accuracy improvement of 1.9\% over the base features across these 49 datasets. The results are presented in descending order based on the performance gap between OpenFE and Base, from largest to smallest  (see more detailed experimental settings in Appendix \ref{app: sec: CC18}). A dash ('–') indicates that the method is not applicable to multi-classification datasets. \xmark \ denotes a failure due to exceeding the runtime limits (24 hours) or bugs in the python package of AutoFeat. The results that demonstrate a significant improvement over others are highlighted in \textbf{bold}. We repeat each experiment 10 times and apply Welch’s t-test with unequal variance and a p-value of 0.05 to assess the significance. n is the number of samples and p is the number of features.}
    \resizebox{1.0\linewidth}{!}{
    \begin{tabular}{lccccccccc}
    \toprule
Dataset & n & p & OpenFE & Base & FCTree & SAFE & AutoFeat & AutoCross & FETCH \\ \midrule
balance-scale & 625 & 4 & \textbf{0.990$_{\pm0.003}$} & 0.860$_{\pm0.005}$ & 0.919$_{\pm0.002}$ & -- & 0.850$_{\pm0.003}$ & -- & 0.965$_{\pm0.004}$ \\ jungle\_chess & 44819 & 6 & \textbf{0.985$_{\pm0.000}$} & 0.864$_{\pm0.000}$ & 0.950$_{\pm0.000}$ & -- & 0.864$_{\pm0.000}$ & -- & 0.950$_{\pm0.000}$ \\ cylinder-bands & 540 & 37 & \textbf{0.865$_{\pm0.014}$} & 0.798$_{\pm0.007}$ & 0.845$_{\pm0.013}$ & 0.805$_{\pm0.021}$ & \xmark & 0.822$_{\pm0.014}$ & 0.733$_{\pm0.022}$ \\ vehicle & 846 & 18 & 0.800$_{\pm0.009}$ & 0.738$_{\pm0.005}$ & 0.725$_{\pm0.008}$ & -- & 0.745$_{\pm0.007}$ & -- & 0.798$_{\pm0.005}$ \\ eucalyptus & 736 & 19 & \textbf{0.722$_{\pm0.013}$} & 0.675$_{\pm0.011}$ & 0.703$_{\pm0.008}$ & -- & 0.665$_{\pm0.010}$ & -- & 0.703$_{\pm0.007}$ \\ mfeat-morphological & 2000 & 6 & \textbf{0.767$_{\pm0.002}$} & 0.730$_{\pm0.000}$ & 0.732$_{\pm0.000}$ & -- & 0.732$_{\pm0.000}$ & -- & 0.739$_{\pm0.002}$ \\ steel-plates-fault & 1941 & 27 & \textbf{0.804$_{\pm0.005}$} & 0.770$_{\pm0.008}$ & 0.786$_{\pm0.007}$ & -- & 0.766$_{\pm0.006}$ & -- & 0.800$_{\pm0.005}$ \\ madelon & 2600 & 500 & \textbf{0.870$_{\pm0.004}$} & 0.842$_{\pm0.004}$ & 0.848$_{\pm0.007}$ & 0.860$_{\pm0.003}$ & 0.845$_{\pm0.006}$ & \xmark &  \xmark \\ climate-model & 540 & 18 & 0.969$_{\pm0.004}$ & 0.943$_{\pm0.006}$ & 0.942$_{\pm0.004}$ & 0.972$_{\pm0.000}$ & 0.943$_{\pm0.006}$ & 0.943$_{\pm0.004}$ & \textbf{0.981$_{\pm0.003}$} \\ analcatdata\_dmft & 797 & 4 & \textbf{0.206$_{\pm0.000}$} & 0.181$_{\pm0.000}$ & 0.181$_{\pm0.000}$ & -- & 0.181$_{\pm0.000}$ & -- & 0.181$_{\pm0.000}$ \\ credit-g & 1000 & 20 & 0.747$_{\pm0.011}$ & 0.726$_{\pm0.012}$ & 0.736$_{\pm0.012}$ & 0.712$_{\pm0.005}$ & 0.734$_{\pm0.007}$ & 0.723$_{\pm0.010}$ & 0.746$_{\pm0.011}$ \\ blood-transfusion & 748 & 4 & \textbf{0.800$_{\pm0.000}$} & 0.780$_{\pm0.000}$ & 0.773$_{\pm0.002}$ & 0.751$_{\pm0.009}$ & 0.780$_{\pm0.000}$ & 0.755$_{\pm0.011}$ & 0.762$_{\pm0.007}$ \\ credit-approval & 690 & 15 & \textbf{0.879$_{\pm0.006}$} & 0.859$_{\pm0.008}$ & 0.871$_{\pm0.004}$ & 0.873$_{\pm0.006}$ & 0.874$_{\pm0.006}$ & 0.867$_{\pm0.008}$ & 0.839$_{\pm0.006}$ \\ mfeat-fourier & 2000 & 76 & \textbf{0.845$_{\pm0.004}$} & 0.825$_{\pm0.008}$ & 0.817$_{\pm0.003}$ & -- & 0.822$_{\pm0.005}$ & -- & 0.815$_{\pm0.003}$ \\ GesturePhase & 9873 & 32 & \textbf{0.688$_{\pm0.002}$} & 0.668$_{\pm0.002}$ & 0.656$_{\pm0.003}$ & -- & 0.666$_{\pm0.003}$ & -- & 0.659$_{\pm0.005}$ \\ pc4 & 1458 & 37 & 0.925$_{\pm0.004}$ & 0.907$_{\pm0.004}$ & 0.912$_{\pm0.005}$ & 0.896$_{\pm0.004}$ & 0.908$_{\pm0.003}$ & 0.915$_{\pm0.006}$ & 0.924$_{\pm0.003}$ \\ pc1 & 1109 & 21 & \textbf{0.941$_{\pm0.003}$} & 0.923$_{\pm0.000}$ & 0.924$_{\pm0.001}$ & 0.927$_{\pm0.002}$ & 0.923$_{\pm0.000}$ & 0.921$_{\pm0.005}$ & 0.924$_{\pm0.002}$ \\ breast-w & 699 & 9 & \textbf{0.991$_{\pm0.005}$} & 0.975$_{\pm0.004}$ & 0.969$_{\pm0.005}$ & 0.979$_{\pm0.004}$ & 0.968$_{\pm0.004}$ & 0.971$_{\pm0.003}$ & 0.966$_{\pm0.003}$ \\ phoneme & 5404 & 5 & 0.917$_{\pm0.002}$ & 0.901$_{\pm0.003}$ & 0.913$_{\pm0.003}$ & 0.916$_{\pm0.002}$ & 0.907$_{\pm0.002}$ & 0.866$_{\pm0.004}$ & 0.904$_{\pm0.002}$ \\ qsar-biodeg & 1055 & 41 & \textbf{0.854$_{\pm0.005}$} & 0.840$_{\pm0.005}$ & 0.837$_{\pm0.004}$ & 0.838$_{\pm0.004}$ & 0.841$_{\pm0.006}$ & 0.845$_{\pm0.003}$ & 0.838$_{\pm0.005}$ \\ pc3 & 1563 & 37 & 0.905$_{\pm0.004}$ & 0.892$_{\pm0.006}$ & 0.894$_{\pm0.006}$ & 0.894$_{\pm0.005}$ & 0.891$_{\pm0.006}$ & 0.904$_{\pm0.004}$ & 0.895$_{\pm0.005}$ \\ mfeat-factors & 2000 & 216 & \textbf{0.960$_{\pm0.003}$} & 0.948$_{\pm0.003}$ & 0.942$_{\pm0.004}$ & -- & 0.942$_{\pm0.002}$ & -- &  \xmark \\ car & 1728 & 6 & \textbf{0.997$_{\pm0.000}$} & 0.986$_{\pm0.000}$ & 0.886$_{\pm0.006}$ & -- & 0.989$_{\pm0.001}$ & -- & 0.970$_{\pm0.004}$ \\ ilpd & 583 & 10 & 0.717$_{\pm0.011}$ & 0.706$_{\pm0.012}$ & 0.691$_{\pm0.008}$ & 0.715$_{\pm0.015}$ & 0.713$_{\pm0.008}$ & 0.708$_{\pm0.016}$ & 0.680$_{\pm0.012}$ \\ electricity & 45312 & 8 & \textbf{0.945$_{\pm0.001}$} & 0.934$_{\pm0.001}$ & 0.940$_{\pm0.001}$ & 0.931$_{\pm0.001}$ & 0.934$_{\pm0.001}$ & 0.930$_{\pm0.001}$ & 0.939$_{\pm0.001}$ \\ kc2 & 522 & 21 & \textbf{0.800$_{\pm0.004}$} & 0.790$_{\pm0.009}$ & 0.794$_{\pm0.009}$ & 0.776$_{\pm0.005}$ & 0.789$_{\pm0.009}$ & 0.789$_{\pm0.009}$ & 0.786$_{\pm0.009}$ \\ spambase & 4601 & 57 & \textbf{0.959$_{\pm0.001}$} & 0.949$_{\pm0.001}$ & 0.947$_{\pm0.002}$ & 0.954$_{\pm0.001}$ & 0.948$_{\pm0.001}$ & 0.951$_{\pm0.002}$ & 0.955$_{\pm0.002}$ \\ ozone-level-8hr & 2534 & 72 & \textbf{0.954$_{\pm0.002}$} & 0.945$_{\pm0.002}$ & 0.946$_{\pm0.002}$ & 0.944$_{\pm0.002}$ & 0.942$_{\pm0.002}$ & 0.942$_{\pm0.001}$ & 0.949$_{\pm0.001}$ \\ semeion & 1593 & 256 & \textbf{0.945$_{\pm0.003}$} & 0.937$_{\pm0.003}$ & 0.939$_{\pm0.003}$ & -- & 0.937$_{\pm0.003}$ & -- &  \xmark \\ Bioresponse & 3751 & 1776 & \textbf{0.791$_{\pm0.005}$} & 0.784$_{\pm0.004}$ & 0.784$_{\pm0.006}$ & 0.784$_{\pm0.004}$ & 0.784$_{\pm0.004}$ & \xmark & \xmark \\ wilt & 4839 & 5 & \textbf{0.986$_{\pm0.001}$} & 0.979$_{\pm0.001}$ & 0.981$_{\pm0.001}$ & 0.984$_{\pm0.001}$ & 0.981$_{\pm0.001}$ & 0.973$_{\pm0.001}$ & 0.984$_{\pm0.002}$ \\ analcatdata\_authorship & 841 & 70 & \textbf{0.989$_{\pm0.002}$} & 0.983$_{\pm0.002}$ & 0.980$_{\pm0.003}$ & -- & 0.983$_{\pm0.002}$ & -- & 0.982$_{\pm0.000}$ \\ mfeat-karhunen & 2000 & 64 & \textbf{0.959$_{\pm0.002}$} & 0.953$_{\pm0.003}$ & 0.950$_{\pm0.002}$ & -- & 0.955$_{\pm0.002}$ & -- & 0.946$_{\pm0.002}$ \\ MiceProtein & 1080 & 77 & \textbf{0.991$_{\pm0.002}$} & 0.985$_{\pm0.003}$ & 0.988$_{\pm0.003}$ & -- & 0.983$_{\pm0.002}$ & -- & 0.980$_{\pm0.003}$ \\ segment & 2310 & 16 & \textbf{0.940$_{\pm0.002}$} & 0.934$_{\pm0.001}$ & 0.935$_{\pm0.001}$ & -- & 0.933$_{\pm0.002}$ & -- & 0.935$_{\pm0.001}$ \\ churn & 5000 & 20 & 0.961$_{\pm0.001}$ & 0.956$_{\pm0.001}$ & 0.959$_{\pm0.001}$ & 0.957$_{\pm0.001}$ & 0.956$_{\pm0.001}$ & 0.944$_{\pm0.001}$ & \textbf{0.969$_{\pm0.001}$} \\ connect-4 & 67557 & 42 & \textbf{0.864$_{\pm0.001}$} & 0.859$_{\pm0.001}$ & 0.858$_{\pm0.001}$ & -- & 0.859$_{\pm0.001}$ & -- & \xmark \\ bank-marketing & 45211 & 16 & \textbf{0.913$_{\pm0.001}$} & 0.909$_{\pm0.001}$ & 0.909$_{\pm0.001}$ & 0.909$_{\pm0.001}$ & 0.910$_{\pm0.001}$ & 0.912$_{\pm0.001}$ & \xmark \\ first-order-theorem-proving & 6118 & 51 & \textbf{0.608$_{\pm0.003}$} & 0.604$_{\pm0.004}$ & 0.605$_{\pm0.002}$ & -- & 0.601$_{\pm0.003}$ & -- & 0.600$_{\pm0.003}$ \\ dna & 3186 & 180 & \textbf{0.959$_{\pm0.002}$} & 0.956$_{\pm0.002}$ & 0.957$_{\pm0.001}$ & -- & 0.956$_{\pm0.002}$ & -- & \xmark \\ numerai28.6 & 96320 & 21 & \textbf{0.517$_{\pm0.001}$} & 0.514$_{\pm0.001}$ & 0.516$_{\pm0.001}$ & 0.514$_{\pm0.001}$ & 0.515$_{\pm0.001}$ & 0.515$_{\pm0.001}$ & \xmark \\ pendigits & 10992 & 16 & \textbf{0.994$_{\pm0.000}$} & 0.991$_{\pm0.001}$ & 0.993$_{\pm0.000}$ & -- & 0.991$_{\pm0.000}$ & -- & 0.993$_{\pm0.001}$ \\ jm1 & 10885 & 21 & 0.815$_{\pm0.003}$ & 0.813$_{\pm0.002}$ & \textbf{0.820$_{\pm0.002}$} & 0.818$_{\pm0.001}$ & 0.814$_{\pm0.002}$ & 0.812$_{\pm0.002}$ & 0.819$_{\pm0.002}$ \\ kr-vs-kp & 3196 & 36 & \textbf{0.996$_{\pm0.001}$} & 0.994$_{\pm0.001}$ & 0.992$_{\pm0.001}$ & 0.994$_{\pm0.001}$ & 0.994$_{\pm0.001}$ & 0.992$_{\pm0.001}$ & 0.994$_{\pm0.001}$ \\ sick & 3772 & 28 & \textbf{0.989$_{\pm0.000}$} & 0.987$_{\pm0.001}$ & 0.986$_{\pm0.002}$ & 0.982$_{\pm0.002}$ & 0.988$_{\pm0.001}$ & 0.982$_{\pm0.002}$ & 0.976$_{\pm0.001}$ \\ wall-robot-navigation & 5456 & 24 & \textbf{1.000$_{\pm0.000}$} & 0.998$_{\pm0.000}$ & 0.999$_{\pm0.000}$ & -- & 0.999$_{\pm0.000}$ & -- & 0.999$_{\pm0.000}$ \\ adult & 48842 & 14 & \textbf{0.875$_{\pm0.001}$} & 0.874$_{\pm0.001}$ & 0.873$_{\pm0.001}$ & 0.872$_{\pm0.001}$ & 0.873$_{\pm0.001}$ & 0.873$_{\pm0.001}$ & \xmark \\ banknote-authentication & 1372 & 4 & \textbf{1.000$_{\pm0.000}$} & 0.999$_{\pm0.002}$ & 0.996$_{\pm0.000}$ & 0.996$_{\pm0.000}$ & 0.995$_{\pm0.002}$ & 0.996$_{\pm0.002}$ & 0.998$_{\pm0.002}$ \\ isolet & 7797 & 617 & \textbf{0.961$_{\pm0.001}$} & 0.960$_{\pm0.001}$ & 0.959$_{\pm0.002}$ & -- & \xmark & -- & \xmark \\ 
\bottomrule
    \end{tabular}
    }
    \label{app: tab: CC18}
\end{table*}

\newpage
\bibliography{icml2023}
\bibliographystyle{icml2023}

\newpage
\appendix
\onecolumn
\input{appendix.tex}

\end{document}

%% file: pdf/runtime.tex
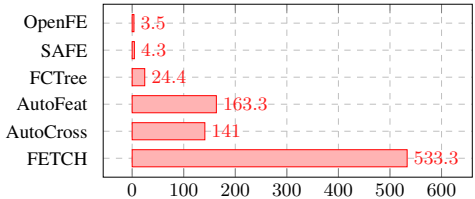
\begin{figure}[t]
    \centering
    \begin{tikzpicture}[scale=0.8]
    \small{
    \begin{axis}[
      ymajorgrids,
      xmajorgrids,
      grid style=dashed,
      xbar,
      height=.26\textwidth,
      width=.45\textwidth,
      bar width=.9em,
      xmax=600,
      xtick={0,100,...,600},
      enlarge y limits=0.14,
      symbolic y coords={{OpenFE},{SAFE},{FCTree},{AutoFeat},{AutoCross},{FETCH}},
      yticklabels={0, OpenFE, SAFE, FCTree, AutoFeat, AutoCross,FETCH},
      ytick distance=1,
      y tick style={opacity=0},
      bar shift=0pt,
      enlarge x limits=0.1,
      xticklabel style={/pgf/number format/fixed,/pgf/number format/fixed zerofill,/pgf/number format/precision=0},
      y dir=reverse,
      nodes near coords,
      nodes near coords align={horizontal}]
      \addplot[fill=red!30, draw=red, text=red!90] coordinates {
          (3.5,{OpenFE})
          (4.3,{SAFE})
          (24.4,{FCTree})
          (163.3,{AutoFeat})
          (141.0,{AutoCross})
          (533.3,{FETCH})
    };
    \end{axis}
  }
    \end{tikzpicture}
    \caption{Runtime comparisons in minutes. The x-axis is the average runtime of TE, DI, NO, VE datasets.}
    \label{fig: runtime}
  \end{figure}

%% file: appendix.tex
\section{Theoretical Results}
\label{app: sec: theory}
\subsection*{Problem Setting}

In this section, we study the advantage of feature generation from a theoretical perspective. In particular, we present a simple yet representative setting in which empirical risk minimization augmented with  feature generation can achieve smaller test loss 
than any learning model without feature generation.

Consider a regression problem with both numerical and categorical features in a transductive learning setting. 
Denote $\numfeature \subseteq \mathbb{R}^d$ as the numerical feature space,
$\catfeature \subseteq \mathbb{N}$ the categorical feature space,
and $\mathcal{Y} \subseteq [0, 1]$ as the target domain. 
We assume $\numfeature$ is $d$-dimensional and convex.
The training data $\Dtrain$ consists of $n$ training data
$\{(X_i,Y_i)\}_{i=1,\ldots, n}$ where 
$X_i=(x_{i0},x_{i1},\ldots, x_{id})$, $x_{i0} \in \catfeature$
is the categorical feature value, $(x_{i1},\ldots, x_{id}) \in \numfeature$
are $d$ numerical feature values.
As a concrete example, one may think each training data as a transaction, 
where $x_{i0}$ as the user id and $(x_{i1},\ldots, x_{id})$ are some features
about the user and the transaction, and $Y$ is the target we want to predict
about the transaction (e.g., lateness of payment, probability of fraudulence).
The test data $\Dtest$ consists of $m$ data points 
$\{(X_i,Y_i)\}_{i=n+1,\ldots, m}$, but these $Y_i$ are the {\em unknown} targets
that we try to predict.

\paragraph{The data model:}
The training set $\Dtrain$ and test set $\Dtest$ are generated
by a two-phase process.
Let $\Delta(\catfeature)$ be the space all probability distributions defined over
$\catfeature$, and $\popunumfeature$ be a probability distribution over  $\Delta(\catfeature)$.
$\Dtrain$ is generated by repeating the following 
$\ktrain$ times:
in the $i$-th iteration, we first take a sample $\Distr_i\in \Delta(\catfeature)$ from $\popunumfeature$ (note that $\Distr_i$ is a distribution over $\catfeature$).
Then, we sample a group of $\nmin$ training data points, each with its categorical feature value being $c_i\in \catfeature$ and its $d$ numerical feature values being
an i.i.d., sample from $\Distr$.
Hence, all these $\nmin$ training data 
$X_{i\nmin+1},\ldots, X_{(i+1)\nmin}$
have the same 
categorical value $c_i\in \catfeature$, and they form the  
$i$-th group $G_i$ (e.g., the set of transactions of the $i$-th user).
$\Dtrain$ contains $\ktrain$ groups $G_{1},\ldots, G_{\ktrain}$.
For the $i$-th data point $(X_i, Y_i)$, we use $\group(i)$ to denote  
the index of the group that contains $(X_i, Y_i)$.
Hence, $(X_i, Y_i)\in G_{\group(i)}$.
Let $\hypothesisclass$ be the hypothesis class
and the true hypothesis is 
$f^*: \numfeature \times \catfeature \rightarrow \mathcal{Y}$ 
such that $Y_{i} = f^*(x_{i1}, \ldots, x_{id}, Z_i)$ where 
$Z_i = \Expt_{X \sim \Distr_{\group(i)}}[X]$ ($Z_i$ is a $d$-dimensional vector)
(e.g., the target depends not only on the numerical feature values of this particular
transaction, but also the mean of the statistics of the user).

The test dataset $\Dtest$ is generated in the same way and it
contains $\ktest$ groups $G_{\ktrain+1},\ldots, G_{\ktrain+\ktest}$.
We assume the categorical feature values
$c_1,\ldots, c_{\ktrain},c_{\ktrain+1}\ldots, c_{\ktrain+\ktest}$ are all distinct
(e.g., a user id that appears in the test set does not appear in the training set).

\paragraph{Feature generation:}
Here we are interested in features generation using operations such as $\mathrm{GroupByThenMean}$. In our setting, the groupby operation is based on the categorical feature $\catfeature$.
So, for a data point $X_i$, we can generate a set of new features
$\hat{X}_i$ which can be computed from all data points in the same group $G_{\group(i)}$.
Formally, let $\setfeaturegen$ be the set of possible feature generation function
and the new feature $\hat{X}_i$ is computed as follows: 
$$
\hat{X}_i =\featuregen(G_{\group(i)}) \text{ for some } \featuregen\in\setfeaturegen.
$$
For a feature generation function $H$ and predictor $f$, the loss on 
the data point $(X_i,Y_i)$ is 
$$
L(H, f, X_i, Y_i)]
=\|Y_i - f(X_i,\hat{X}_i) \|^2
=\|Y_i - f(X_i,\featuregen(G_{\group(i)})) \|^2.
$$
Our goal is to find a feature generation function $\featuregen\in\setfeaturegen$
and a predictor $\hat{f}\in \hypothesisclass$ such that the test loss
is minimized
$$
L_{\Dtest}(H, \hat{f}) 
=\Expt_{(X_i,Y_i)\in \Dtest} [L(H, \hat{f}, X_i, Y_i)]
=\frac{1}{m}\sum_{(X_i,Y_i)\in \Dtest}\|Y_i - f(X_i,\featuregen(G_{\group(i)})) \|^2.
$$
If we do not use any feature generation, the loss of the predictor $f'$ 
\footnote{
Note that the input dimension for the predictor $f'$ here is different since there is no
generated feature. But this is not an issue for our following argument.
}
is simply
$$
L_{\Dtest}(f') = \Expt_{(X_i,Y_i)\in \Dtest}[\|Y_i - f'(X_i) \|^2]
$$

In the following, we show that under nature conditions,
we can achieve vanishing test loss by using feature generation (as the number of training samples $n$ and the minimum size of each group $\nmin$ become larger). See Theorem~\ref{thm:learnable}. But if we only use the raw feature $X_i$ for predicting $Y_i$, the expected test loss of any predictor is at least a positive constant. See Theorem~\ref{thm:notlearnable}.

\subsection*{Learnability with Feature Generation}



For a particular feature generation function $H$,
we use $\mathrm{Rad}_k(\mathcal{F})$ to denote the empirical Rademacher complexity of $\hypothesisclass$ over $k$ random samples: 
$$
\mathrm{Rad}_k(\mathcal{F})
=
\Expt_{\{(X_i, Y_i)\}_{i=1}^{k}} \frac{1}{k} \Expt_{\sigma}\bigg[\sup_{f\in\mathcal{F}}\sum_{i=1}^{k} \sigma_i L(H,f, X_i, Y_i) \bigg]  
$$
where $\sigma = (\sigma_1, \cdots, \sigma_{k})$ are independent Rademacher variables and $X_i$ is an i.i.d. sample from $\Distr_i$, which is an i.i.d. sample from $\popunumfeature$, and $Y_{i} = f^*(X_i, Z_i)$ where 
$Z_i = \Expt_{X \sim \Distr_{i}}[X]$.
We assume that $\mathrm{Rad}_{\ktrain}(\mathcal{F}) \rightarrow 0$
as $\ktrain\rightarrow\infty$
(for many hypothesis classes, $\mathrm{Rad}_{\ktrain}$ scales as $O(\sqrt{1/\ktrain})$) \citep{foundationbook}.
We also assume any function $f(\cdot,\cdot) \in \mathcal{F}$ is Lipschitz on $z$: There exists constant $C_{\mathcal{F}}$ such that $|f(X,Z_1) - f(X,Z_2)| \leq C_{\mathcal{F}}\|Z_1 - Z_2\|$ for any $z_1, z_2, x \in \mathcal{X}$ and $f \in \mathcal{F}$. 
We further assume there exists constant $B_{\mathcal{X}}$ such that $\sup_{x \in \mathcal{X}}\|x\| \leq B_{\mathcal{X}}$.

\begin{theorem}
\label{thm:learnable}
There is a feature generation function $H$, such that 
the test loss of the empirical risk minimizer $\hat{f}$ can be bounded by 
$$
L_{\Dtest}(H, \hat{f}) \leq
2 \mathrm{Rad}_{\ktrain}(\mathcal{F}) + 
\sqrt{2  \ln(4\delta^{-1})/\ktrain} + 
6 B_{\mathcal{X}} C_{\mathcal{F}} \sqrt{2d  \ln (4d(\ktrain + \ktest)\delta^{-1}) / h}
$$
with probability at least $1 - \delta$.
In particular, the test loss approaches to $0$ 
when $\ktrain, \nmin \rightarrow \infty$ . 
\end{theorem}

\begin{proof}
We fix the feature generation function to be GroupByThenMean. In concrete, we have $\hat{X}_i = \featuregen(G_{\group(i)}) = \frac{1}{|G_{\group(i)}|}\sum_{j \in G_{\group(i)}} X_j.$
The algorithm then find $f$ that minimizes the emprical risk:
$$\hat{f} = \arg \min_{f \in \mathcal{F}} L_{\Dtrain}(H, f).$$

According to Hoeffding’s inequality and union bound for each dimension, 
with probability at least $1 - \delta/(2(\ktrain + \ktest))$, we have $$\|\widehat{X}_i - Z_i\| = \Big\|\frac{1}{|G_{\group(i)}|}\sum_{j \in G_{\group(i)}} X_j - \Expt_{X \sim \Distr_{\group(i)}}[X]\Big\| \leq B_{\mathcal{X}} \sqrt{2d  \ln (4d(\ktrain + \ktest)\delta^{-1}) / h},$$
for some concrete $\group(i)$. By applying union bound on all groups, this statement holds for all $i$ with probability at least $1 - \delta/2$.

In this case, for any function $f \in \mathcal{F}$, we have 
\begin{align}\label{eqn:upper-bound-proof-statement-1}
    & \Big| (f(X_i, Z_i) - Y_i)^2 - (f(X_i, \hat{X}_i) - Y_i)^2 \Big| \notag\\
    \leq\; & |f(X_i, Z_i) - f(X_i, \hat{X}_i)| \cdot \Big|f(X_i, Z_i) - Y_i + f(X_i, \hat{X}_i) - Y_i \Big| \notag \\
    \leq\; & 2 B_{\mathcal{X}} C_{\mathcal{F}}  \sqrt{2d  \ln (4d(\ktrain + \ktest)\delta^{-1}) / h}
\end{align}
where the last inequality holds since $|f(X_i, Z_i) - f(X_i, \hat{X}_i)| \leq C_{\mathcal{F}} \|\widehat{z}_i - z_i\|$ and $\mathcal{Y} \subseteq [0,1]$.
Define $H^*$ be the function $H^*(G_{\group(i)}) = Z_i= \Expt_{X \sim \Distr_{\group(i)}}[X]$.
We note this statement further implies $\Big|L_{\Dtrain}(H^*, f) - L_{\Dtrain}(H, f) \Big| \leq \sqrt{2d  \ln (4d(\ktrain + \ktest)\delta^{-1}) / h}$ and $\Big|L_{\Dtest}(H^*, f) - L_{\Dtest}(H, f) \Big| \leq \sqrt{2d  \ln (4d(\ktrain + \ktest)\delta^{-1}) / h}$ according to the definition of the loss function.

Note that each data point is not an i.i.d. sample (since two points in one group 
are correlated), we need to define the following group Rademacher complexity
to bound the generalization error.
We define {\em group Rademacher complexity} to be Rademacher complexity but  defined over groups:
\begin{align*} 
    \textrm{Rad}^G_{\ktrain}(\mathcal{F}) := \Expt_{\Dtrain} \frac{1}{\ktrain} \Expt_{\sigma}\bigg[\sup_{f\in\mathcal{F}}\sum_{i=1}^{\ktrain} \sigma_i \cdot \frac{1}{h}\sum_{j \in G_i} L(H,f, X_j, Y_j) \bigg],
\end{align*}
where $\sigma = (\sigma_1, \cdots, \sigma_{\ktrain})$ are independent Rademacher variables.
Note that if we view each group as one random sample,
these groups are i.i.d. samples.
Hence, we can apply the classical generalization bound via Rademacher complexity \citep{foundationbook}, which asserts that with probability at least $1 - \delta/2$ for any function $f \in \mathcal{F}$ the following holds
\begin{align}\label{eqn:upper-bound-proof-statement-2}
    \Big|L_{\Dtest}(H, f) - L_{\Dtrain}(H, f) \Big| &\leq 2 \textrm{Rad}^G_{\ktrain}(\mathcal{F}) + \sqrt{2 \ln(4\delta^{-1})/\ktrain}. 
\end{align}
Moreover, one can see the group Rademacher complexity can be upper bounded by the ordinary Rademacher complexity for i.i.d. samples:
\begin{align} \label{eqn:upper-bound-proof-statement-2s} 
    \textrm{Rad}^G_{\ktrain}(\mathcal{F}) &\leq \frac{1}{h}  \sum_{i=1}^{h} \Expt_{\Dtrain} \frac{1}{\ktrain} \Expt_{\sigma}\bigg[ \sup_{f\in\mathcal{F}}\sum_{j=1}^{\ktrain} \sigma_j \cdot L(H,f, X_{G_{i,j}}, Y_{G_{i,j}}) \bigg] =\mathrm{Rad}_{\ktrain}(\mathcal{F})
\end{align}
where $G_{i,j}$ is the $j$-th element in $G_i$.

According to union bound, both \eqref{eqn:upper-bound-proof-statement-1} and \eqref{eqn:upper-bound-proof-statement-2} are satisfied for all $f \in \mathcal{F}$ with probability at least $1-\delta$.
In this case, since $f^*$ is the ground truth, its empirical risks satisfies $L_{\Dtrain}(H^*, f^*) = 0$. In this case, we have 
\begin{align} 
    L_{\Dtrain}(H, \hat{f}) &\leq L_{\Dtrain}(H, f^*) \notag \\
    &\leq \Big|L_{\Dtrain}(H, f^*) - L_{\Dtrain}(H^*, f^*) \Big| + L_{\Dtrain}(H^*, f^*) \notag \\
    & \leq 2 B_{\mathcal{X}} C_{\mathcal{F}}  \sqrt{2d  \ln (4d(\ktrain + \ktest)\delta^{-1}) / h} \label{eqn:upper-bound-proof-statement-3}
\end{align}
where the first inequality dues to empirical risk minimization and the last inequality dues to \eqref{eqn:upper-bound-proof-statement-1}.
As a result, the population loss can then be bounded by
\begin{align*}
    L_{\Dtest}(H, \hat{f}) &\leq \Big|L_{\Dtest}(H, \hat{f}) - L_{\Dtest}(H^*, \hat{f}) \Big| + \Big|L_{\Dtest}(H^*, \hat{f}) - L_{\Dtrain}(H^*, \hat{f}) \Big| \\
    & \qquad + \Big| L_{\Dtrain}(H^*, \hat{f}) - L_{\Dtrain}(H, \hat{f}) \Big| + L_{\Dtrain}(H, \hat{f}) \\
    &\leq 2 B_{\mathcal{X}} C_{\mathcal{F}} \sqrt{2d  \ln (4d(\ktrain + \ktest)\delta^{-1}) / h} + 2 \textrm{Rad}^G_{\ktrain}(\mathcal{F}) + \sqrt{2 \ln(4\delta^{-1})/\ktrain} \\
    & \qquad + 2 B_{\mathcal{X}} C_{\mathcal{F}} \sqrt{2d  \ln (4d(\ktrain + \ktest)\delta^{-1}) / h} + 2 B_{\mathcal{X}} C_{\mathcal{F}} \sqrt{2d  \ln (4d(\ktrain + \ktest)\delta^{-1}) / h}\\ 
    &\leq 2 \mathrm{Rad}_{\ktrain}(\mathcal{F}) + 
\sqrt{2  \ln(4\delta^{-1})/\ktrain} + 
6 B_{\mathcal{X}} C_{\mathcal{F}} \sqrt{2d  \ln (4d(\ktrain + \ktest)\delta^{-1}) / h}
\end{align*}
where the second inequality holds due to \eqref{eqn:upper-bound-proof-statement-2} and \eqref{eqn:upper-bound-proof-statement-3} while the last inequality follows from \eqref{eqn:upper-bound-proof-statement-2s}.
This proves the desired statement.
\end{proof}

\noindent
{\bf Remarks:} 
We only consider the case where $\setfeaturegen$
only contain one particular operation $\mathrm{GroupByThenMean}$.
In fact, it is possible to extend the above setting to one with multiple feature generation functions. Here we require that such feature generation functions
be statistics of the corresponding group and can be estimated using i.i.d. samples.
In our two-phase generative process, each group contains exactly $\nmin$ data points. 
We can easily extend our theorem to the setting where
the size of each group is also a random variable as long as it takes value
at least $\nmin$ with high probability. 
The main idea is the same but the notations would become very tedious. 
Since our goal here is to illustrate
the statistical advantage of feature generation, we choose to present a simplified yet representative setting.

\subsection*{Without Feature Generation}

\begin{theorem}
\label{thm:notlearnable}
In case that we do not use any feature generation, there exists a problem instance such that, no matter how large $\ktrain$  (number of groups in the training set), $\ktest$ ((number of groups in the test set)), and $\nmin$ (the size of each group) are, for any function $f': \mathcal{X}\rightarrow \mathcal{Y}$, the test loss is at least 
\begin{align*}
    L_{\Dtest}(f') \geq \frac{3}{64}.
\end{align*}
\end{theorem}

\begin{proof}
Consider a problem instance with $\mathcal{X} = \mathcal{Y} = [0, 1]$. 
The data distribution is generated in the following way: 
Each group contains $\nmin$ data points and is generated in the following way: The distribution of $i$-th group $\Distr_{\group(i)}$ is random between $B(\frac{3}{4})$ and $B(\frac{1}{4})$ with equal probability where $B(p)$ is the Bernoulli distribution such that  $\Pr[B(p)=1]=p=1-\Pr[B(p)=0]$. Thus, $Z_i=\Expt_{X \sim \Distr_{\group(i)}}[X]$ is either $\frac{3}{4}$ or $\frac{1}{4}$ with equal probability for each group. 
$X_i$ is 0/1 random variable, generated i.i.d. from either $B(\frac{3}{4})$
if $Z_i=\frac{3}{4}$ and from either $B(\frac{1}{4})$ if $Z_i=\frac{1}{4}$.
We assume the target of data point is given by $Y_{i} = f^*(X_i, Z_i) = Z_i$. 


When $X_i=1$, the probability of $Y_i=\frac{1}{4}$ is given by
$$\Pr\Big[Y_i=\frac{1}{4}\Big|X_i=1\Big] = \Pr\Big[Z_i=\frac{1}{4}\Big|X_i=1\Big] = \frac{\Pr[Z_i=\frac{1}{4}, X_i=1]}{\sum_{z} \Pr[Z_{i} = z]\Pr[X_i=1| Z_i=z]} = \frac{1}{4}.$$
Similarly, the probability of $Y_i=\frac{3}{4}$ is $\Pr[Y_i=\frac{3}{4}|X_i=1] = \frac{3}{4}.$
So for any predictor $f': \mathcal{X}\rightarrow \mathcal{Y}$, we have $$\Expt_{X_i=1}[\|Y_i - f'(X_i) \|^2] = \frac{1}{4} \cdot \Big(f'(X_i) - \frac{1}{4}\Big)^2 + \frac{3}{4} \cdot \Big(f'(X_i) - \frac{3}{4}\Big)^2 \geq \frac{3}{64}.$$
The last inequality holds because the left hand side is a quadratic function.
With the same reasoning, we have $\Expt_{X_i=0}[\|Y_i - f'(X_i) \|^2] \geq \frac{3}{64}$. As a result, the test loss of for any predictor $f'$ is at least 
$$L_{\Dtest}(f') =  \Expt_{(X_i,Y_i)\in \Dtest}[\|Y_i - f'(X_i) \|^2]\geq \frac{3}{64}.$$
The above argument does not depend on how large $\ktrain$, $\ktest$ and $\nmin$
are.
\end{proof}

\section{Data}
\label{app: sec: data}
\begin{table}[ht]
\centering
\caption{Datasets description}
\renewcommand\arraystretch{1.3}
\resizebox{\textwidth}{!}{%
    \begin{tabular}{lcccccccc}
        \toprule
        Name & Abbr & \# Train & \# Validation & \# Test & \# Num & \# Cat & \# Ord & Task type \\
        \midrule
    
        California Housing~\citep{california_housing} & CA & $13209$ & $3303$ & $4128$ & $7$ & $0$ & $1$ & Regression \\
        Microsoft~\citep{microsoft} & MI & $723412$ & $235259$ & $241521$ & $111$ & $0$ & $25$ & Regression \\
        Medical & ME & $104361$ & $26091$ & $32613$ & $5$ & $6$ & $0$ & Regression \\
        Telecom & TE & $32669$ & $8168$ & $10210$ & $23$ & $22$ & $12$ & Binclass \\
        Broken Machine & BR & $576000$ & $144000$ & $180000$ & $58$ & $0$ & $0$ & Binclass \\
        Diabetes~\citep{diabetes} & DI & $65129$ & $16283$ & $20354$ & $3$ & $34$ & $10$ & Binclass \\
        Nomao~\citep{nomao} & NO & $22465$ & $6000$ & $6000$ & $34$ & $29$ & $55$ & Binclass \\
        Vehicle~\citep{vehicle} & VE & $60000$ & $18528$ & $20000$ & $100$ & $0$ & $0$ & Binclass \\
        Jannis~\citep{jannis} & JA & $53588$ & $13398$ & $16747$ & $54$ & $0$ & $0$ & Multiclass \\
        Covertype~\citep{covertype} & CO & $371847$ & $92962$ & $116203$ & $9$ & $0$ & $45$ & Multiclass \\
    
        \bottomrule
    \end{tabular}
}
\label{app: tab: data}
\end{table}
We describe the details of the datasets in Table \ref{app: tab: data}. The reason we did not perform evaluation on the datasets used by previous studies~\citep{ExploreKit,NFS,DBLP:conf/automl/DIFER,FETCH} is that, these datasets only contain a few hundreds to a few thousand data points with a few dozens of features. Experimental results on toy datasets are hardly convincing for the real progress in feature generation. The datasets used in our benchmark range in size from moderate to big, with some having millions of samples and hundreds of features.

\section{Additional Results}
\begin{table}[ht]
\centering
\caption{Comparison between OpenFE and other baselines. The results of baseline methods are from the corresponding papers. Our results are averaged by 10 different random seeds.}
\renewcommand\arraystretch{1.4}
\begin{tabular}{cccccccccc} 
\toprule
Dataset & Source & C$\backslash$ R & Instances$\backslash$Features & Random & TransGraph & LFE & NFS & OpenFE \\ [0.5ex] 
\midrule
Airfoil & UCIrvine & R & 1503$\backslash$5 & 0.753 & 0.801 & - & 0.796 & $\mathbf{0.808}$\\ 
German Credit & UCIrvine & C & 1000$\backslash$24 & 0.655 & 0.724  & - & 0.805 & $\mathbf{0.815}$ \\ 
Higgs Boson Subset & UCIrvine & C & 50000$\backslash$28 & 0.699 & 0.729  & 0.68 & 0.731 & $\mathbf{0.741}$ \\ 
Ionosphere & UCIrvine & C & 351$\backslash$34 & 0.934 & 0.941  & 0.932 & 0.969 & $\mathbf{0.986}$ \\
SpamBase & UCIrvine & C & 4601$\backslash$57 & 0.937 & $\mathbf{0.961}$  & 0.947 & 0.955 & $\mathbf{0.961}$ \\
SpectF & UCIrvine & C & 467$\backslash$44 & 0.748 & 0.788  & - & 0.876 & $\mathbf{0.877}$ \\
Sonar & UCIrvine & C & 208$\backslash$60 & 0.723 & - & 0.801 & 0.839 & $\mathbf{0.929}$ \\
\bottomrule
\end{tabular}
\label{tab: additional_baselines}
\end{table}
\subsection{Comparisons with Other Baselines}
\label{app: sec: comparison_with_other_baselines}
We compare OpenFE with other baselines, including some learning-based methods that lack critical details for code reproduction. These mehtods include:
\begin{itemize}
    \item \textbf{Random}~\citep{RL}. Randomly include features from candidate feature set multiple times and select new features with improvement according to CV scores. Simple methods can be powerful in AutoML~\cite{naive_automl}.
    \item \textbf{TransGraph}~\citep{RL}. TransGraph uses reinforcement learning to traverse a transformation graph for feature transformations.
    \item \textbf{LFE}~\citep{LFE}. LFE recommends feature transformations by meta learning approaches.
    \item \textbf{NFS}~\citep{NFS}. NFS uses a recurrent neural network controller to search for a series of transformations.
\end{itemize}
Following previous studies~\citep{NFS}, the metric for regression datasets is $1-(\text{relative absolute error})$ and the metric for classification datasets is F1-score. We present the results in Table \ref{tab: additional_baselines}. OpenFE also surpasses other baseline methods in these datasets.

\begin{table}[h]
\centering
\caption{Results in Table \ref{tab: nn_results} with standard deviations.}
\begin{tabular}{clccccc}
\toprule
Model & Features & CA $\downarrow$ & MI $\downarrow$ & NO $\uparrow$ & VE $\uparrow$ & JA $\uparrow$ \\ \midrule
AutoInt & Base & 0.474$_{\pm0.004}$ & 0.750$_{\pm0.001}$ & 0.993$_{\pm0.000}$ & 0.927$_{\pm0.001}$ & 0.720$_{\pm0.002}$ \\ 
AutoInt & OpenFE & \textbf{0.462$_{\pm0.005}$} & \textbf{0.745$_{\pm0.000}$} & \textbf{0.994$_{\pm0.000}$} & \textbf{0.929$_{\pm0.001}$} & \textbf{0.724$_{\pm0.003}$} \\ \midrule
DCN-V2 & Base & 0.483$_{\pm0.002}$ & 0.749$_{\pm0.001}$ & 0.993$_{\pm0.000}$ & 0.924$_{\pm0.001}$ & 0.716$_{\pm0.002}$ \\
DCN-V2 & OpenFE & \textbf{0.472$_{\pm0.004}$} & \textbf{0.743$_{\pm0.000}$} & \textbf{0.994$_{\pm0.000}$} & \textbf{0.926$_{\pm0.000}$} & 0.716$_{\pm0.001}$ \\ \midrule
TabNet & Base & 0.510$_{\pm0.005}$ & 0.751$_{\pm0.001}$ & 0.993$_{\pm0.000}$ & 0.924$_{\pm0.001}$ & 0.724$_{\pm0.004}$ \\ 
TabNet & OpenFE & \textbf{0.501$_{\pm0.009}$} & \textbf{0.746$_{\pm0.001}$} & \textbf{0.995$_{\pm0.000}$} & \textbf{0.926$_{\pm0.001}$} & 0.724$_{\pm0.004}$ \\ \midrule
NODE & Base & 0.464$_{\pm0.002}$ & 0.745$_{\pm0.000}$ & 0.994$_{\pm0.000}$ & 0.927$_{\pm0.000}$ & 0.727$_{\pm0.002}$ \\ 
NODE & OpenFE & \textbf{0.457$_{\pm0.002}$} & \textbf{0.740$_{\pm0.000}$} & \textbf{0.995$_{\pm0.001}$} & \textbf{0.929$_{\pm0.000}$} & \textbf{0.731$_{\pm0.001}$} \\ \midrule
FT-T & Base & 0.459$_{\pm0.003}$ & 0.746$_{\pm0.000}$ & 0.993$_{\pm0.000}$ & 0.927$_{\pm0.001}$ & 0.733$_{\pm0.002}$ \\ 
FT-T & OpenFE & \textbf{0.453$_{\pm0.003}$} & \textbf{0.741$_{\pm0.001}$} & \textbf{0.995$_{\pm0.000}$} & \textbf{0.929$_{\pm0.000}$} & \textbf{0.738$_{\pm0.002}$} \\ \bottomrule
\end{tabular}
\label{app: tab: nn_with_std}
\end{table}

\subsection{Results with Standard Deviations}
\label{app: sec: standard deviations}
We present the results of the effect of OpenFE on DNNs with standard deviations in Table \ref{app: tab: nn_with_std}.

\subsection{Additional Experiments on OpenML CC18 Benchmark}
\label{app: sec: CC18}
We expand our evaluation using the curated OpenML CC18 benchmarking suites~\citep{DBLP:conf/nips/OpenML-CC18}, which consist of 68 tabular datasets. The  Three conclusions can be drawn from the additional experimental results:

\begin{enumerate}
    \item Feature generation leads to statistically significant improvements over base features in learning performance on 49 out of 68 datasets.
    \item OpenFE outperforms all baseline methods significantly on 40 out of these 49 datasets.
    \item OpenFE provides an average accuracy improvement of 1.9\% over the base features across these 49 datasets.
\end{enumerate}

We present the results in Table \ref{app: tab: CC18}.

\textbf{Regarding the experimental setup:} For each dataset, we use the same 0.6/0.2/0.2 split (train/val/test) for all methods. The evaluation metric across all datasets is accuracy. LightGBM is chosen as the modeling algorithm for the tabular data, and we tune its hyperparameters using the validation set and the base features. We repeat each experiment 10 times and apply Welch’s t-test with unequal variance and a p-value of 0.05 to assess the significance.

\textbf{Regarding standard deviation:} The standard deviation of results are generally small because LightGBM training is typically robust to different random seeds, ensuring consistent performance across multiple trials. In the table below, the standard deviation shown as 0.000 indicates that it is less than 0.0005.



\begin{table}[]
    \centering
    \caption{Comparison between the results of FCTree, SAFE, and AutoCross from their paper and our reproduced results to verify the reproduction.}
    
    \begin{tabular}{lccc}
        \toprule
         & Metric & FCTree (from paper) & FCTree (reproduced) \\
        \midrule
        Transfusion  & Accuracy & $0.752$ & $0.793\pm 0.017$ \\
        Nuclear & AUC & $0.629$ & $0.625\pm 0.009$ \\
        \midrule
         & Metric & SAFE (from paper) & SAFE (reproduced) \\
        \midrule
        Magic & AUC & $0.9288$ & $0.9370\pm 0.0009$ \\
        Spambase & AUC & $0.9846$ & $0.9837\pm 0.0012$ \\
        \midrule
                 & Metric & AutoCross (from paper) & AutoCross (reproduced) \\
        \midrule
        Bank & AUC & $0.9455$ & $0.9456\pm 0.0008$ \\
        Adult & AUC & $0.9280$ & $0.9251\pm 0.0003$ \\
        Credit & AUC & $0.8567$ & $0.8624\pm 0.0002$ \\
        \bottomrule
    \end{tabular}
    \label{app: tab: reproduce}
\end{table}

\subsection{Reproduction}
\label{app: sec: reproduce}
We reproduce FCTree, SAFE, and AutoCross according to the descriptions in their papers. In order to make sure that we reproduce a reasonable version of these methods, we compare the results of our reproduced methods with the ones in their paper. We present the results in Table \ref{app: tab: reproduce}. We can see that most of the results of the reproduced methods closely match or even outperform the results from the papers.

\begin{table}[ht]
\centering
\caption{The running time of different methods in minutes. IEEE: IEEE-CIS Fraud Detection, BNP: BNP Paribas Cardif Claims Management.}
\renewcommand\arraystretch{1.2}
\begin{tabular}{lcccccccccccc}
    \toprule
    {} & CA  & MI & ME & TE & BR & DI & NO  & VE & JA & CO & IEEE & BNP \\
    \midrule
    FCTree
    & 2.3 & 1378 & 6.9 & 5.7 & 330 & 6.8 & 11 & 74 & 40 & 160 & - & -\\
    SAFE
    & - & - & - & 5 & 6.0 & 0.9 & 1.3 & 10 & - & - & - & -\\
    AutoFeat 
    & 0.2 & 23 & 20.6 & 32.1 & 537 & 37 & 49 & 535 & 354 & 1284 & - & -\\
    AutoCross
    & - & - & - & 101 & 1078 & 169 & 148 & 146 & - & - & - & -\\
    FETCH
    & 98.1 & - & 1202 & 150 & - & 241 & 325 & 1417 &  528 & - & - & -
    \\ \midrule
    OpenFE         
    & 0.1 & 31 & 1.5 & 2.1 & 4.7 & 4 & 4.7 & 3.3 & 3.5 & 20 & 92 & 0.9\\
    \bottomrule
\end{tabular}
\label{app: tab: running time}
\end{table}

\subsection{Running Time}
\label{app: sec: running time}
We present the running time of different methods in Table \ref{app: tab: running time}. The experimental environment is the same for all the methods. We can see that OpenFE is a fast method that terminates within a reasonable amount of time even for large datasets.

\begin{table}[ht]
\centering
\caption{Comparisons between MDI, permutation, and SHAP in feature importance attribution.}
\renewcommand\arraystretch{1.4}
\begin{tabular}{cccc}
\toprule
     & MDI  & permutation & SHAP \\ \midrule
Rank & $2.07 \scriptscriptstyle \pm \scriptstyle 0.93$ & $1.79 \scriptscriptstyle \pm \scriptstyle 0.91$        & $2.14 \scriptscriptstyle \pm \scriptstyle 0.69$ \\
Time & 0s   & 25min       & 42s  \\ \bottomrule
\end{tabular}
\label{tab: ablation}
\end{table}

\subsection{Comparing Feature Importance Attribution Methods}
\label{app: sec: feature importance}

In this section, we compare MDI, permutation, and SHAP in feature importance attribution in OpenFE. We use the results of OpenFE on benchmarking datasets to rank these methods. The training model is LightGBM. We present the results in Table \ref{tab: ablation}. We also present the running time of each method on the Microsoft dataset, the benchmarking dataset with largest number of samples. Different feature attribution methods do not differ much on most of the datasets. MDI can be obtained for free after the training process of LightGBM, while permutation and SHAP may require longer running time, depending on the sizes of datasets.

\section{Implementation Details}
\label{app: sec: implementation_details}
\subsection{Experimental Environment}
\label{app: sec: environment}
For all the experiments, feature generation is carried out on a workstation with  Intel(R) Xeon(R) Gold 6230 CPU @ 2.10GHz, 40 cores, 512G memory. Model tuning and model training are performed on one or more NVidia Tesla V100 16Gb.

\begin{table}[h]
\setlength\tabcolsep{20pt}
\centering
\caption{Unary operators.}
\begin{tabular}{cc}\toprule
            Numerical & Categorical \\\midrule
             $\begin{aligned}
    & \mathrm{Freq}, \mathrm{Abs}, \mathrm{Log}, \\ & \mathrm{Sqrt}, \mathrm{Sigmoid}, \\ & \mathrm{Round}, \mathrm{Residual}
\end{aligned}$ & $\mathrm{Freq}$
            \\\bottomrule
\end{tabular}
\label{app: tab: unary-op}
\end{table}

\begin{table}[h]
\setlength\tabcolsep{10pt}
\centering
\caption{Binary operators. (N refers to Numerical, C refers to Categorical)}
\begin{tabular}{ccc}\toprule
            N $\times$ N & N $\times$ C & C $\times$ C \\\midrule
$\begin{aligned}
    & \min, \max, \\ & +, -,  \times ,  \div 
\end{aligned}$ &
 $\begin{aligned}
    & \mathrm{GroupByThenMin},  \mathrm{GroupByThenMax},\\
    & \mathrm{GroupByThenMean}, \mathrm{GroupByThenMedian}, \\
    & \mathrm{GroupByThenStd}, \mathrm{GroupByThenRank}
\end{aligned}$ &
$\begin{aligned}
    & \mathrm{Combine},\\&  \mathrm{CombineThenFreq}, \\ & \mathrm{GroupByThenNUnique}
\end{aligned}$
            \\\bottomrule
\end{tabular}
\label{app: tab: binary-op}
\end{table}

\subsection{Operators and Feature Transformations}
\label{app: sec: operators}
All the operators are classified into unary operators (Table \ref{app: tab: unary-op}) and binary operators (Table \ref{app: tab: binary-op}), where unary operators act on one feature and binary operators act on two features. Then the operators are further classified according to the type of features they act on. For example, GroupByThenMean requires a categorical feature and a numerical feature, while Max requires two numerical features. All the features in a dataset are classified into numerical features, categorical features, and ordinal features. The difference between ordinal features and categorical features is that an ordinal feature has a clear ordering of categories (such as ``age''). Ordinal features are treated as both numerical and categorical when generating features transformations. For example, we can calculate GroupByThenMean(age, gender), which is the average age of each gender. We can also calculate GroupByThenMean(income, age), which is the average income of people of different ages. For anonymized datasets where the meanings of features are unknown, features with string values are treated as categorical features. Features with discrete values (the number of unique values is less than 100) are treated as ordinal features. Features with continuous values are treated as numerical features.

We enumerate all the first-order transformations to form the candidate feature set. A first-order transformation uses one operator once to transform base features. For example, $\text{weight}/\text{height}$ is a first-order transformation of base features $\text{weight}$ and $\text{height}$. $\text{BMI}=\text{weight}/\text{height}^2$ is a second-order transformation.

We list the details of all the operators below. $f$ stands for a numerical feature and $c$ stands for a categorical feature.

\begin{itemize}
    \item $\mathrm{Freq}(f)$. The frequency of feature $f$.
    \item $\mathrm{Freq}(c)$. The frequency of feature $c$.
    \item $\mathrm{Abs}(f)$. Element-wise absolute value.
    \item $\mathrm{Log}(f)$. Element-wise logarithm.
    \item $\mathrm{Sqrt}(f)$. Element-wise square root.
    \item $\mathrm{Sigmoid}(f)$. Element-wisely apply function $x \mapsto \frac{1}{1+e^{-x}}$.
    \item $\mathrm{Round}(f)$. Element-wise rounding.
    \item $\mathrm{Residual}(f)$. Element-wisely take decimal part.
    \item $\mathrm{Min}(f_1, f_2)$. Element-wise minimum.
    \item $\mathrm{Max}(f_1, f_2)$. Element-wise maximum.
    \item $f_1+f_2$. Element-wise addition.
    \item $f_1-f_2$. Element-wise subtraction.
    \item $f_1 \times f_2$. Element-wise multiplication.
    \item $f_1 \div f_2$. Element-wise division.
    \item $\mathrm{GroupByThenMin}(f, c)$. The minimum value of $f$ in each category of feature $c$.
    \item $\mathrm{GroupByThenMax}(f, c)$. The maximum value of $f$ in each category of feature $c$.
    \item $\mathrm{GroupByThenMean}(f, c)$. The average value of $f$ in each category of feature $c$.
    \item $\mathrm{GroupByThenMedian}(f, c)$. The median of $f$ in each category of feature $c$.
    \item $\mathrm{GroupByThenStd}(f, c)$. The standard deviation of $f$ in each category of feature $c$.
    \item $\mathrm{GroupByThenRank}(f, c)$. The ranking of $f$ in each category of feature $c$. The rankings are normalized between 0 and 1.
    \item $\mathrm{Combine}(c_1, c_2)$. Comebine the categories in feature $c_1$ and $c_2$ to be new categories. For example, for a data point, ``vocation'' is ``doctor'' and ``hobby'' is ``football'', then the value for $\mathrm{Combine}(\text{vocation}, \text{hobby})$ of this data point is a new category of ``$\text{doctor-}\text{football}$''.
    \item $\mathrm{CombineThenFreq}(c_1,c_2)$. Same as $\mathrm{freq}\left( \mathrm{Combine}(c_1,c_2) \right)$.
    \item $\mathrm{GroupByThenNUnique}(c_1,c_2)$. The number of unique values of $c_1$ in each category of feature $c_2$.
\end{itemize}



\begin{table}[ht]
\caption{The parameters of OpenFE for each dataset. IEEE: IEEE-CIS Fraud Detection. BNP: BNP Paribas Cardif Claims Management.}
\centering
\renewcommand\arraystretch{1.3}
\begin{tabular}{ccccccccccccc}
\toprule
                & CA & MI & ME & TE & BR & DI & NO & VE & JA & CO & IEEE & BNP\\ \midrule
\# data blocks  & 1  & 64 & 1 & 1 & 1  & 8  & 32 & 16 & 1  & 8 & 32 & 8\\
\# new features & 10 & 10 & 10 & 10 & 10 & 10 & 10 & 10 & 50 & 50 & 600 & 200\\
\bottomrule
\end{tabular}
\label{app: tab: parameters}
\end{table}

\subsection{Parameters}
\label{app: sec: parameters}
We present the hyperparameters of OpenFE for each dataset in Table \ref{app: tab: parameters}. 
When the number of data blocks is one, we use all the data to calculate and evaluate new features in successive featurewise pruning, and eliminate new features with negative reduction in loss. For most datasets, the number of new features is set to be 10. The effective new features are usually sparse in the vast pool of candidate new features. For two multi-classification datasets, the number of new features is set to be 50. For a fair comparison, all the baseline methods include the same number of new features, except for the \textbf{NN} baseline. For the \textbf{NN} baseline, the number of new features is the number of hidden units in the last hidden layer, which is determined by hyperparameter tuning.

In OpenFE, the default parameter of LightGBM in successive featurewise pruning has 1000 number of estimators, 0.1 learning rate, 16 leaves, and 3 early stopping rounds. The default parameter in feature importance attribution is the same except for 50 early stopping rounds.

\subsection{Feature Importance Attribution Methods}
MDI is the gain importance embedded in LightGBM. For permutation feature importance, we randomly shuffle the values of features on the validation set, and observe the change in validation loss. For SHAP feature importance, we average the SHAP values of each sample for each feature.


\section{Discussion}
\subsection{Discovered Feature Transformations}
\label{app: sec: discovered feature transformations}

We discuss the effective feature transformations discovered by OpenFE for the DI (Diabetes) dataset. The DI dataset collects 10 years (1999-2008) of clinical care at 130 US hospitals, where the goal is to predict the readmission~\citep{diabetes}. One can see from Table \ref{tab: baselines_comparison} that OpenFE outperforms other baseline methods and improves the learning performance by a great margin. The top-1 feature discovered by OpenFE is freq(patient\_id), which is the number of times the patient has been admitted to the hospital and is highly predictive of whether the patient will be readmitted to the hospital. However, other methods may fail to find this new feature since the feature patient\_id itself is not useful. For example, SAFE~\citep{SAFE} only generates candidate features based on the base features that are used as splits in XGBoost. However, since patient\_id itself is not useful, it will not be used for splits in XGBoost. This example also demonstrates that, reducing the number of candidate features by heuristic assumptions may risk missing useful candidate features.

\subsection{High-order Features}
\label{app: sec: high-order}
How to search for high-order feature transformations is challenging in automated feature generation due to the explosion in search space~\citep{NFS}. Some previous methods argue that high-order features are useful by directly searching for high-order features~\citep{NFS,RL}. However, the effectiveness of high-order feature transformations should be evaluated in light of all its low-order components. For example, a second-order feature transformation $f_1\times f_2\times f_3$ is effective only if it has additional effectiveness to all their first-order components $f_1\times f_2$, $f_1\times f_3$, and $f_2\times f_3$. In addition, because high-order feature transformations are typically more difficult to interpret than low-order ones, low-order feature transformations are favoured in industrial applications where interpretability is important. In a word, searching for high-order features in a hierarchical manner is more appropriate than directly searching for high-order features in two aspects: 1) evaluating the effectiveness of high-order features in a more accurate way. 2) generating low-order features with better interpretability.

Whether it is necessary to generate high-order features is case-by-case. Because the search space of high-order features is usually incredibly huge (even if we limit the order), none of the existing methods can enumerate all the high-order features within reasonable computational resources. We can hardly claim that high-order features are not useful for a dataset. However, we do not find that generating high-order features is useful for all the benchmarking datasets in our experiments. In the IEEE competition, generating high-order features does not seem to be useful for neither Expert nor OpenFE. In the BNP competition, generating high-order features provides a small improvement in the test score. We may conclude that first-order features are usually more important than high-order ones in feature generation.

\section{Complexity Analysis}
\label{app: sec: complexity}
\noindent \textbf{Complexity of Generating Base Predictions.} Let $n$ be the number of samples, $m$ be the number of base features of dataset, and $k$ be the number of folds. Complexity of generating base predictions is $k$ times of GBDT\footnote{In complexity analysis, we refer to the implementation of LightGBM. The dominant term in complexity is building histogram, i.e., $O(nm)$ per depth per tree.} training, hence 
$$
O\left(kTdnm\right),
$$ 
where each GBDT contains $T$ trees with a maximum depth $d$.

\noindent \textbf{Complexity of STAGE I.} Suppose we split the dataset into $2^q$ data blocks, the number of candidate features is $m^2$. There are $2^{-q} m^2$ features remaining after successive featurewise halving. The complexity is $\sum_{i=0}^{q} {2^{(i-q)}}\cdot{m}^2\cdot C(2^{-i}n, 1)$, where $C(n,m)$ is the complexity of GBDT training with data shape $n\times m$. If the GBDT contains $T_1$ trees with a maximum depth of $d_1$, we have the time complexity $$O\left(2^{-q}qT_1d_1nm^2\right).$$

\noindent \textbf{Complexity of STAGE II.} The complexity of stage II is dominated by a single GBDT training. Suppose the GBDT has $T_2$ trees with a maximum depth of $d_2$, then the time complexity of stage II is $$O\left(2^{-q}T_2d_2nm^2\right).$$

\noindent \textbf{Overall Complexity.}  In the implementation of OpenFE, the number of trees and the maximum depth of GBDT are predefined fixed constant. If we regard $T$ and $d$ as constant, the overall complexity of OpenFE is $$O(2^{-q}qnm^2).$$